\documentclass{article}

\usepackage{journal}
\usepackage{natbib}
\usepackage[mathscr]{eucal}
\usepackage[usenames,dvipsnames]{xcolor}
\usepackage[backref,pageanchor=true,plainpages=false,pdfpagelabels,bookmarks,bookmarksnumbered,breaklinks,pdfborder={0 0 0},colorlinks=true,
	citecolor=Sepia,
	linkcolor=Sepia,
	filecolor=Sepia,
	urlcolor=Sepia
]{hyperref}

\usepackage{framed}
\usepackage{amsmath,amsfonts,amsbsy,amsgen,amscd,mathrsfs,amssymb}
\usepackage[utf8]{inputenc} 
\usepackage[T1]{fontenc}    
\usepackage{hyperref}       
\usepackage{booktabs}       
\usepackage{nicefrac}       
\usepackage{microtype}      
\usepackage[framemethod=TikZ]{mdframed}
\mdfsetup{skipabove=\topskip,skipbelow=\topskip}

\usepackage{bm} 
\usepackage{bbm}
\usepackage{xspace}
\usepackage[algo2e,ruled,linesnumbered,vlined]{algorithm2e}
\usepackage{tcolorbox}

\usepackage{amsthm}
\usepackage{mathtools}
\usepackage[shortlabels]{enumitem}

\makeatletter
\def\set@curr@file#1{\def\@curr@file{#1}} 
\makeatother
\usepackage[load-configurations=version-1]{siunitx} 

\newtheorem{thm}{Theorem}
\newtheorem{lem}[thm]{Lemma}

\newtheorem{conj}{Conjecture}
\newtheorem{claim}[thm]{Claim}

\theoremstyle{definition}
\newtheorem{defn}{Definition}
\newtheorem*{defn*}{Definition}
\newtheorem{exam}{Example}
\newtheorem{prop}[thm]{Proposition}
\renewenvironment{proof}[1][]{\par\noindent{\bf Proof #1\ }}{\hfill\BlackBox\\}

\usepackage{prettyref}
\newrefformat{fig}{Figure~\ref{#1}}
\newrefformat{alg}{Algorithm~\ref{#1}}
\newrefformat{def}{Definition~\ref{#1}}
\newrefformat{eqn}{Equation~\ref{#1}}
\newrefformat{cor}{Corollary~\ref{#1}}
\newrefformat{app}{Appendix~\ref{#1}}
\newrefformat{rem}{Remark~\ref{#1}}
\newrefformat{assu}{Assumption~\ref{#1}}
\newrefformat{q}{Question~\ref{#1}}
\newrefformat{clm}{Claim~\ref{#1}}
\newrefformat{prop}{Proposition~\ref{#1}}
\newrefformat{exam}{Example~\ref{#1}}

\jmlrheading{}{2022}{}{}{}{O. Montasser, S. Hanneke \& N. Srebro}
\ShortHeadings{Montasser Hanneke Srebro}{Complexity of Robust Learnability}
\newcommand{\inbrace}[1]{\left \{ #1 \right \}}
\newcommand{\inparen}[1]{\left ( #1 \right )}
\newcommand{\insquare}[1]{\left [ #1 \right ]}

\newcommand{\abs}[1]{\left\lvert #1 \right\rvert}

\newlength{\dhatheight}

\mathchardef\mhyphen="2D

\newcommand{\SET}[1]{\inbrace{#1}}
\newcommand{\setdef}[2]{\SET{#1 : #2}}

\DeclareMathOperator*{\argmin}{argmin}

\DeclareMathOperator*{\Ex}{\mathbb{E}}

\DeclareMathOperator*{\Prob}{Pr}

\newcommand{\eps}{\varepsilon}

\newcommand{\bbN}{{\mathbb N}}

\newcommand{\bbR}{{\mathbb R}}

\newcommand{\bbA}{{\mathbb A}}
\newcommand{\bbB}{{\mathbb B}}
\newcommand{\bbG}{{\mathbb G}}

\let\boldm\bm

\newcommand{\by}{{\boldm y}}

\newcommand{\calD}{\mathcal{D}}
\newcommand{\calE}{\mathcal{E}}

\newcommand{\calH}{\mathcal{H}}

\newcommand{\calL}{\mathcal{L}}
\newcommand{\calM}{\mathcal{M}}

\newcommand{\calO}{\mathcal{O}}

\newcommand{\calU}{\mathcal{U}}

\newcommand{\calX}{\mathcal{X}}
\newcommand{\calY}{\mathcal{Y}}

\newcommand{\ERM}{\textsf{ERM}\xspace}

\newcommand{\ind}{\mathbbm{1}}
\newcommand{\Risk}{{\rm R}}

\newcommand{\vc}{{\rm vc}}

\newcommand{\MAJ}{{\rm MAJ}}

\newcommand{\Mre}{m_0}

\newcommand{\wrt}{with respect to }

\newcommand{\RE}{{\rm RE}}
\newcommand{\re}{{\rm re}}
\newcommand{\ag}{{\rm ag}}
\newcommand{\outdeg}{{\rm adv\mhyphen outdeg}}
\newcommand{\adeg}{{\rm advdeg}}

\newcommand{\CM}{{\mathfrak{D}}_{\calU}(\calH)}
\newcommand{\graph}{global one-inclusion graph}

\newcommand{\removed}[1]{}

\usepackage{time}

\begin{document}
\title{Adversarially Robust Learning:\\ A Generic Minimax Optimal Learner and Characterization}
\author{%
 \name{Omar Montasser} \email{omar@ttic.edu}\\
 \addr Toyota Technological Institute at Chicago\\
 \name{Steve Hanneke} \email{steve.hanneke@gmail.com}\\
 \addr Purdue University\\
 \name{Nathan Srebro} \email{nati@ttic.edu}\\
 \addr Toyota Technological Institute at Chicago
}

\maketitle

\begin{abstract}
We present a minimax optimal learner for the problem of learning predictors robust to adversarial examples at test-time. Interestingly, we find that this requires new algorithmic ideas and approaches to adversarially robust learning. In particular, we show, in a strong negative sense, the suboptimality of the robust learner proposed by \citet*{pmlr-v99-montasser19a} and a broader family of learners we identify as \emph{local} learners. Our results are enabled by adopting a \emph{global} perspective, specifically, through a key technical contribution: the \emph{\graph}, which may be of independent interest, that generalizes the classical one-inclusion graph due to \citet*{DBLP:journals/iandc/HausslerLW94}. Finally, as a byproduct, we identify a dimension characterizing qualitatively and quantitatively what classes of predictors $\calH$ are robustly learnable. This resolves an open problem due to \citet{pmlr-v99-montasser19a}, and closes a (potentially) infinite gap between the established upper and lower bounds on the sample complexity of adversarially robust learning. 
\end{abstract}

\section{Introduction}
\label{sec:intro}
We study the problem of learning predictors that are {\em robust} to adversarial examples at test-time. Adversarial examples can be thought of as carefully crafted input perturbations that cause predictors to misclassify. Learning predictors robust to adversarial examples is a major contemporary challenge in machine learning. There has been a significant interest lately in how deep learning predictors are {\em not} robust to adversarial examples \citep[][]{szegedy2013intriguing,biggio2013evasion,DBLP:journals/corr/GoodfellowSS14} -- e.g., to adversarial perturbations of bounded $\ell_p$-norms-- leading to an ongoing effort to devise methods for learning predictors that {\em are} adversarially robust.

The aim of this paper is to put forward a theory precisely characterizing the complexity of \emph{robust} learnability. We know from prior work that finite VC dimension is \emph{sufficient} for robust learnability, but we also know that its finiteness is not \emph{necessary} \citep*{pmlr-v99-montasser19a}. Furthermore, there is a (potentially) \emph{infinite} gap between the established quantitative upper and lower bounds on the sample complexity of adversarially robust learning \citep*{pmlr-v99-montasser19a}, and we do not know of any \emph{optimal} learners for this problem. In this paper, we address the following fundamental questions:
\begin{center}
    {\em What classes of predictors $\calH$ are \emph{robustly} learnable \wrt an \emph{arbitrary} perturbation set?\\
    Can we design generic optimal learners for adversarially robust learning?}
\end{center}

The problem of characterizing learnability is the most basic question of statistical learning theory. In classical (non-robust) supervised learning, the fundamental theorem of statistical learning \citep*{vapnik:71,vapnik:74,blumer:89,ehrenfeucht:89} provides a complete understanding of \emph{what} is learnable: classes $\calH$ with finite VC dimension, and \emph{how} to learn: by the generic learner \emph{empirical risk minimization} ($\ERM_\calH$). We also know that $\ERM_\calH$ is a near-optimal learner for $\calH$ with sample complexity that is quantified tightly by the VC dimension of $\calH$.

\paragraph{Problem setup.} Given an instance space $\calX$ and label space $\calY=\SET{\pm 1}$, we consider robustly learning an arbitrary hypothesis class $\calH\subseteq \calY^\calX$ \wrt an arbitrary perturbation set $\calU:\calX\to 2^\calX$, where $\calU(x)\subseteq \calX$ represents the set of perturbations that can be chosen by an adversary at test-time, as measured by the \emph{robust risk}:
\begin{equation}
    \label{eqn:rob-risk}
    \Ex_{(x,y)\sim\calD} \insquare{\sup_{z\in \calU(x)}\ind\SET{h(z)\neq y}}.
\end{equation}
We denote by ${\rm RE}(\calH,\calU)$ the set of distributions $\calD$ over $\calX\times \calY$ that are \emph{robustly realizable}: $\exists h^*\in \calH, \Risk_\calU(h^*;\calD)=0$. A learner $\bbA:(\calX\times \calY)^*\to \calY^\calX$ receives $n$ i.i.d.~examples $S=\SET{(x_i,y_i)}_{i=1}^{n}$ drawn from some unknown distribution $\calD\in {\rm RE}(\calH,\calU)$, and outputs a predictor $\bbA(S)$. The worst-case expected \emph{robust} risk of learner $\bbA$ \wrt $\calH$ and $\calU$ is defined as:
\begin{equation}
\label{eqn:exp-rob-risk-A}
    \calE_n(\bbA;\calH,\calU)=\sup_{\calD\in {\rm RE(\calH,\calU)}} \Ex_{S\sim \calD^{n}} \Risk_\calU(\bbA(S);\calD).
\end{equation}
The \emph{minimax} expected robust risk of learning $\calH$ \wrt $\calU$ is defined as:
\begin{equation}
\label{eqn:exp-rob-risk}
    \calE_n(\calH,\calU)=\inf_{\bbA} \calE_n(\bbA;\calH,\calU).
\end{equation}
For any $\eps \in (0,1)$, the \emph{sample complexity of realizable robust $\eps$-PAC learning of $\calH$ with respect to $\calU$}, denoted $\calM^\re_{\eps}(\calH,\calU)$, is defined as 
\begin{equation}
\label{eqn:sample-complexity}
\calM^{\re}_\eps(\calH,\calU) = \min\SET{n\in \bbN\cup\SET{\infty}: \calE_{n}(\calH,\calU) \leq \eps}.
\end{equation}
$\calH$ is robustly PAC learnable realizably with respect to $\calU$ if $\forall_{\epsilon\in (0,1)}$, $\calM^{\re}_\eps(\calH,\calU)$ is finite. 

\paragraph{Related work and gaps.}\citet*{pmlr-v99-montasser19a} showed that any class $\calH$ with finite VC dimension is robustly PAC learnable \wrt any perturbation set $\calU$; by establishing that $\calM^{\re}_\eps(\calH,\calU)\leq \Tilde{O}(\frac{2^{\vc(\calH)}}{\eps})$, where $\vc(\calH)$ denotes the VC dimension of $\calH$. While this gives a \emph{sufficient} condition for robust learnability, they also showed that finite VC dimension is {\em not} necessary for robust learnability, indicating a (potentially) infinite gap between the established upper and lower bounds on the sample complexity. We next provide simple motivating examples that highlight these gaps in this existing theory, and suggest that the learner witnessing this upper bound might be very sub-optimal:

\begin{exam}
\label{exam:all-labels}
{\normalfont Consider an infinite domain $\calX$, the hypothesis class of all possible predictors $\calH = \calY^{\calX}$, and an all-powerful perturbation set $\calU(x)=\calX$. In this case, the hypothesis minimizing the population robust risk $\Risk_{\calU}(h;\calD)$ would always be the all-positive or the all-negative hypothesis, and so these are the only two hypotheses we should compete with.  And so, even though $\vc(\calH)=\infty$, a single example suffices to inform the learner of whether to produce the all-positive or all-negative function.} 
\end{exam}

\begin{exam}
\label{exam:halfspaces}
{\normalfont A less extreme and more natural example is to take $\calX = \bbR^{\infty}$ (an infinite dimensional space), and $\calH$ the set of homogeneous halfspaces in $\calX$, and a perturbation set $\calU(x)=\{z\in\calX: \langle x, v\rangle = \langle z, v\rangle \text{ for }v\in V\}$ where $V$ is the set of the first $d$ standard basis vectors. In this example, an adversary is allowed to \emph{arbitrarily} corrupt all but $d$ features. Note that $\vc(\calH)=\infty$ but we can robustly PAC learn $\calH$ with $O(d)$ samples: simply project samples from $\calX$ onto the subspace spanned by $V$ and learn a $d$-dimensional halfspace.}
\end{exam}

\paragraph{Our contributions.} In fact, even more strongly, we show in \prettyref{thm:local-alg-lower} that there are problem instances $(\calH,\calU)$ that are {\em not} robustly learnable by the learner proposed by \citet*{pmlr-v99-montasser19a}, but {\em are} robustly learnable with a different generic learner. Beyond this, \prettyref{thm:local-alg-lower} actually illustrates, in a strong negative sense, the suboptimality of any \emph{local} learner -- a family of learners that we identify in this work -- which informally {\em only} has access to labeled training examples and perturbations of the training examples, but otherwise does not know the perturbation set $\calU$ (defined formally in \prettyref{def:local-alg}). 

In this work, we adopt a \emph{global} perspective on robust learning. In \prettyref{sec:graph}, we introduce a novel graph construction, the \emph{\graph}, that in essence embodies the complexity of \emph{robust} learnability. In \prettyref{thm:optimal}, for any class $\calH$ and perturbation set $\calU$, we utilize the \graph~to construct a generic \emph{minimax optimal} learner $\bbG_{\calH,\calU}$ satisfying $\calE_{2n-1}(\bbG_{\calH,\calU};\calH,\calU)\leq 4\cdot \calE_{n}(\calH,\calU)$. Our \graph~utilizes the structure of the class $\calH$ and the perturbation set $\calU$ in a global manner by considering \emph{all} datasets of size $n$ that are robustly realizable, where each dataset corresponds to a vertex in the graph. Edges in the graph correspond to pairs of datasets that agree on $n-1$ datapoints, disagree on the $n^{\text{th}}$ label, and \emph{overlap} on the $n^{\text{th}}$ datapoint according to their $\calU$ sets. We arrive at an optimal learner by orienting the edges of this graph to minimize a notion of \emph{adversarial-out-degree} that corresponds to the average leave-one-out \emph{robust} error. Our learner avoids the lower bound in \prettyref{thm:local-alg-lower} since it is \emph{non-local} and utilizes the structure of $\calU$ at \emph{test-time}.

In \prettyref{sec:complexity-measure}, we introduce a new complexity measure denoted $\CM$ (defined in \prettyref{eqn:complexity-measure}) based on our \graph. We show in \prettyref{thm:realizable-qual} that $\CM$ \emph{qualitatively} characterizes robust learnability: a class $\calH$ is robustly learnable \wrt $\calU$ if and only if $\CM$ is finite. In \prettyref{thm:realizable-bounds}, we show that $\CM$ tightly \emph{quantifies} the sample complexity of robust learnability: $\Omega(\frac{\CM}{\eps})\leq \calM^{\re}_\eps(\calH,\calU) \leq \Tilde{O}(\frac{\CM}{\eps})$. This closes the (potentially) \emph{infinite} gap previously established \citep*{pmlr-v99-montasser19a}.

In \prettyref{sec:agnostic}, beyond the realizable setting, we show in \prettyref{thm:agnostic} that our complexity measure $\CM$ bounds the sample complexity of \emph{agnostic} robust learning: $\calM^{\ag}_{\eps}(\calH,\calU)\leq \Tilde{O}(\frac{\CM}{\eps^2})$. This shows that $\CM$ tightly (up to log factors) characterizes the sample complexity of \emph{agnostic} robust learning, since by definition, $\calM^{\ag}_{\eps}(\calH,\calU)\geq \calM^{\re}_{\eps}(\calH,\calU)$.

\section{Local learners are suboptimal}
\label{sec:local}
In this section, we identify a broad family of learners, which we term \emph{local} learners, and show that such learners are \emph{suboptimal} for adversarially robust learning. Informally, \emph{local} learners {\em only} have access to labeled training examples and perturbations of the training examples, but otherwise do not know the perturbation set $\calU$. More formally, 
 
\begin{defn} [Local Learners]
\label{def:local-alg}
For any class $\calH$, a local learner $\bbA_\calH:(\calX\times \calY \times 2^\calX)^*\to \calY^\calX$ for $\calH$ takes as input a sequence $S_\calU=\SET{(x_i,y_i,\calU(x_i)}_{i=1}^{m} \in \calX\times \calY\times 2^\calX $ consisting of labeled training examples and their corresponding perturbations according to some perturbation set $\calU$, and outputs a predictor $f\in\calY^\calX$. In other words, $\bbA$ has full knowledge of $\calH$, but only \emph{local} knowledge of $\calU$ through the training examples.
\end{defn}
We note that the robust learner proposed by \citet*{pmlr-v99-montasser19a}, for example, {\em is} a local learner: for a given a class $\calH$ and input $S_\calU=\SET{(x_i,y_i,\calU(x_i)}_{i=1}^{m}$, their learner outputs a majority-vote over predictors in $\calH$, that are carefully chosen based on the input $S_\calU$. Moreover, adversarial training methods in practice \citep[e.g., ][]{DBLP:conf/iclr/MadryMSTV18,DBLP:conf/icml/ZhangYJXGJ19} are also examples of local learners (they only utilize the perturbations on the training examples). We provably show next that \emph{local} learners are \emph{not} optimal. We give a construction where it is not possible to robustly learn without taking advantage of the information about $\calU$ at \emph{test-time}. 
\begin{thm} 
\label{thm:local-alg-lower}
There is an instance space $\calX$ and a class $\calH$, such that for any \emph{local} learner $\bbA_\calH:(\calX\times \calY\times 2^\calX)^*\to \calY^\calX$ and any sample size $m\in \bbN$, there exists a perturbation set $\calU$ for which:
\begin{enumerate}
    \item $\bbA_\calH$ fails to robustly learn $\calH$ \wrt $\calU$ using $m$ samples.
    \item There exists a non-local learner $\bbG_{\calH,\calU}:(\calX\times \calY)^*\to\calY^\calX$ which robustly learns $\calH$ \wrt $\calU$ with $0$ samples.
\end{enumerate}
\end{thm}
This negative result highlights that there are limitations to what can be achieved with {\em local} learners. It also highlights the importance of utilizing the structure of the perturbation set $\calU$ at test-time, which has been observed in the context of \emph{transductive} robust learning where the learner receives a training set of $n$ labeled examples and a test set of $n$ unlabeled adversarial perturbations, and is asked to label the test set with few errors \citep*{montasser2021transductive}. In practice, randomized smoothing \citep{DBLP:conf/icml/CohenRK19} is an example of a non-local method in the sense that at prediction time, it uses the perturbation set to compute predictions.
\begin{proof}[of \prettyref{thm:local-alg-lower}]
We begin with describing the instance space $\calX$ and the class $\calH$. Pick three infinite unique sequences $(x^{+}_n)_{n\in \bbN}$, $(x^{-}_n)_{n\in \bbN}$, and $(z_n)_{n\in \bbN}$ from $\bbR^2$ such that for each $n \in \bbN: x^+_n=(n,1), x^-_n=(n,-1), z_n=(n,0)$, and let $\calX= \cup_{n\in \bbN} \SET{x^+_n, x^-_n, z_n}$. Consider the class $\calH$ defined by 
\begin{equation}
\label{eqn:class-construction}
    \calH=\SET{h_\by: \by \in \SET{\pm 1}^{\bbN}}, \text{ where }h_\by(z_n)=y_n \wedge h_\by(x_n^+)=+1 \wedge h_\by(x_n^-)=-1 \inparen{\forall n\in \bbN}.
\end{equation}
Observe that all classifiers in $\calH$ are constant on  $(x^{+}_n)_{n\in \bbN}$ and $(x^{-}_n)_{n\in \bbN}$, but they shatter $(z_n)_{n\in \bbN}$. 
We will consider a \emph{random} perturbation set $\calU:\calX\to 2^\calX$ that is defined as follows:
\begin{equation}
\label{eqn:random-U}
    \forall n\in\bbN: 
    \begin{cases}
        \calU(x^+_{n}) = \SET{x_n^+,z_n}\text{ and }\calU(x^-_{n}) = \SET{x_n^-}\text{ and }\calU(z_{n}) = \SET{x_n^+,x_n^-,z_n}\text{w.p. }\frac{1}{2},\\
        \calU(x^+_{n}) = \SET{x_n^+}\text{ and }\calU(x^-_{n}) = \SET{x_n^-, z_n}\text{ and }\calU(z_{n}) = \SET{x_n^+,x_n^-,z_n}~~\text{w.p. }\frac{1}{2}.
    \end{cases}
\end{equation}

For any sample size $m\in\bbN$, let $P$ be a uniform distribution on $$\SET{(x^{+}_1,+1),(x_1^-,-1),\dots, (x_{3m}^+, +1), (x_{3m}^{-},-1)}.$$
Observe that for any randomized $\calU$ (according to \prettyref{eqn:random-U}), the distribution $P$ is \emph{robustly realizable} \wrt $\calU$: $\exists h\in \calH, \Risk_\calU(h;P)=0$. 
Let $\bbA$ be an arbitrary \emph{local} learner (see \prettyref{def:local-alg}), i.e., $\bbA$ has full knowledge of the class $\calH$, but only partial knowledge of $\calU$ through the training samples.
Let $S\sim P^{m}$ be a fixed random set of training examples drawn from $P$. Then,
\begin{align*}
    \Ex_{\calU} \Risk_{\calU}(\bbA(S_\calU);P) &= \Ex_{\calU} \Ex_{(x,y)\sim P} \ind[\exists z\in \calU(x): \bbA(S_\calU)(z)\neq y]\\
    &\geq \Prob_{(x,y)\sim P}\insquare{(x,y)\notin S} \Ex_{\calU} \Ex_{(x,y)\sim P} \insquare{\ind[\exists z\in \calU(x): \bbA(S_\calU)(z)\neq y] | (x,y)\notin S}\\
    &= \Prob_{(x,y)\sim P}\insquare{(x,y)\notin S} \Ex_{(x,y)\sim P} \Prob_{\calU}\insquare{\exists z\in \calU(x): \bbA(S_\calU)(z)\neq y] | (x,y)\notin S}\\
    &\geq \frac{1}{3}\cdot \frac{1}{2} = \frac{1}{6}.
\end{align*}
By law of total expectation, this implies that there exists a deterministic choice of $\calU$ such that $\Ex_{S\sim P^{m}}\Risk_\calU(\bbA(S_\calU);P)\geq \frac{1}{6}$. This establishes that $\bbA$ fails to robustly learn $\calH$ \wrt $\calU$ using $m$ samples. On the other hand, $\calH$ is robustly learnable \wrt $\calU$ with $0$ samples by means of our non-local learner $\bbG_{\calH,\calU}$ (see \prettyref{sec:optimal} and \prettyref{thm:optimal}) which utilizes \emph{full} knowledge of $\calU$. In particular, $0$ samples are needed, since the graph $G^{\calU}_{\calH}$ will contain \emph{no} edges by the definition of $\calH$ (\prettyref{eqn:class-construction}) and $\calU$ (\prettyref{eqn:random-U}).
\end{proof}

\section{A global one-inclusion graph}
\label{sec:graph}
To go beyond the limitations of local learners from \prettyref{sec:local}, in this section, we introduce: the \emph{\graph}, the main mathematical object which allows us to adopt a global perspective on robust learning.\removed{enables our results, use to answer Questions~\ref{q:one}, \ref{q:two}, and \ref{q:three}} Our \graph~is inspired by the classical one-inclusion graph introduced by \citet*{DBLP:journals/iandc/HausslerLW94}, which leads to an algorithm that is near-optimal for (non-robust) PAC learning, and has also been adapted and used in multi-class learning \citep*{DBLP:conf/nips/RubinsteinBR06, DBLP:conf/colt/DanielyS14, DBLP:journals/corr/abs-2203-01550} and for learning partial concept classes\footnote{At a first glance, it might seem that adversarially robust learning can be viewed as a special case of learning partial concepts classes \citep{DBLP:conf/focs/AlonHHM21}, but we would like to remark that this is {\em not} the case. The apparent similarity arises because it is possible to state the robust realizability assumption in the language of partial concept classes, as in the example mentioned in \citet{DBLP:conf/focs/AlonHHM21} on learning linear separators with a margin, but this is just an assumption on the data-distribution. Specifically, a partial concept class learner is {\em only} guaranteed to make few errors on samples drawn from the distribution \citep[see Definition 2 in][]{DBLP:conf/focs/AlonHHM21}, and not on their adversarial perturbations: i.e., performance is still measured under 0-1 loss, not robust risk.} \citep*{DBLP:conf/focs/AlonHHM21}. Before introducing our \graph, to ease the readers, we begin first with describing the construction of the classical one-inclusion graph due to \citet*{DBLP:journals/iandc/HausslerLW94} and its use as a (non-robust) learner, and discuss its limitations for adversarially robust learning.

\subsection{Background: classical one-inclusion graphs}
For a given class $\calH$ and a finite dataset $X=\SET{x_1,\dots,x_n}\subseteq \calX$, the classical one-inclusion graph $G_{X,\calH}$ consists of vertices $V=\SET{(h(x_1),\dots,h(x_n)): h\in \calH}$ where each vertex $v=(h(x_1),\dots,h(x_n))\in V$ is a \emph{realizable} labeling of $X$, and two vertices $u,v\in V$ are connected with an edge if and only if they differ only in the labeling of a single $x_i\in X$. \citet*{DBLP:journals/iandc/HausslerLW94} showed that the edges in $G_{X,\calH}$ can be oriented such that each vertex has out-degree at most $\vc(\calH)$. Now, how can the one-inclusion graph be used as a learner? Given a training set of examples $S=\SET{(x_1,y_1),\dots,(x_{n-1},y_{n-1})}$ and a test example $x_n$, we construct the one-inclusion graph on $\SET{x_1,\dots,x_{n-1}}\cup\SET{x_n}$ using the class $\calH$ and orient it such that maximum out-degree is at most $\vc(\calH)$. Then, we use the orientation to predict the label of the test point $x_n$. Specifically, if there exists $h,h'\in\calH$ such that $\forall 1\leq i \leq n-1: h(x_i)=h'(x_i)$ and $h(x_n)\neq h'(x_n)$ then we will have two vertices in the graph $v=(h(x_1),\dots,h(x_{n-1}),h(x_n))$ and $u=(h(x_1),\dots,h(x_{n-1}),h'(x_n))$ with an edge connecting them (because they differ only in the label of $x_n$), and we predict the label of $x_n$ that this edge is directed towards. Since each vertex has out-degree at most $\vc(\calH)$, this implies that the average leave-one-out-error (which bounds the expected risk from above) is at most $\frac{\vc(\calH)}{n}$. 

What breaks in the adversarial learning setting? At test-time, we do not observe an i.i.d. test example $x\sim \calD$ but rather only an adversarially chosen perturbation $z\in\calU(x)$. This completely breaks the exchangeability analysis of the classical one-inclusion graph, because the training points are i.i.d.~but the perturbation $z$ of the test point $x$ is not. Furthermore, the classical one-inclusion graph is \emph{local} in the sense that it depends on the training data and the test point, and as such different perturbations $z,\Tilde{z}\in\calU(x)$ could very well lead to different graphs, different orientations, and ultimately different predictions for $z$ and $\Tilde{z}$ which by definition imply that the prediction is not robust on $x$. 

\subsection{Our global one-inclusion graph}
We now describe the construction of the global one-inclusion graph. For any class $\calH$, any perturbation set $\calU$, and any dataset size $n\in \bbN$, denote by $G^\calU_\calH=(V_n,E_n)$ the \graph. In words, $V_n$ is the collection of all datasets of size $n$ that can be \emph{robustly} labeled by class $\calH$ \wrt perturbation set $\calU$. Formally, each vertex $v\in V_n$ is represented as a \emph{multiset} of labeled examples $(x,y)$ of size $n$:\footnote{Note that we allow a labeled example $(x,y)$ to appear more than once in a vertex $v$, hence the multiset representation.}
\begin{equation}
\label{eqn:oneinclusion-vertices-v2}
    V_n = \SET{\SET{(x_1,y_1),\dots,(x_n,y_n)}: \inparen{\exists h\in \calH} \inparen{\forall i\in[n]}\inparen{\forall z\in \calU(x_i)}, h(z)=y_i}.
\end{equation}
Two vertices (datasets) $u,v \in V_n$ are connected by an edge if and only if there is a unique labeled example $(x,y)\in v$ that does not appear in $u$ and there is a unique labeled example $(\Tilde{x},\Tilde{y})\in u$ that does not appear in $v$ satisfying: $y\neq \Tilde{y}$ and $\calU(x)\cap \calU(\Tilde{x})\neq \emptyset$. Formally, $u,v\in V_n$ are connected by an edge if and only if their symmetric difference $u\Delta v =\SET{(x,y),(\Tilde{x},\Tilde{y})}$ where $y\neq \Tilde{y}$ and $\calU(x)\cap \calU(\Tilde{x})\neq \emptyset$. Furthermore, we will label an edge by the perturbation $z\in \calU(x)\cap \calU(\Tilde{x})$ that witnesses this edge:
\begin{equation}
\label{eqn:oneinclusion-edges-v2}
    E_n = \setdef{ (\SET{u,v}, z)}{u,v\in V_n \wedge u\Delta v = \SET{(x,y),(\Tilde{x},\Tilde{y})} \wedge y\neq \Tilde{y} \wedge z\in \calU(x)\cap\calU(\Tilde{x})}.
\end{equation}
For each vertex $v\in V_n$, denote by $\adeg(v)$ the adversarial degree of $v$ which is defined as the number of elements $(x,y)\in v$ that witness an edge incident on $v$:
\begin{equation}
\label{eqn:deg-v2}
\adeg(v) = \abs{\SET{(x,y)\in v:\exists u \in V_n, z\in \calX\text{ s.t. }(\SET{v, u}, z)\in E_n\wedge (x,y)\in v\Delta u}}.
\end{equation}
We want to emphasize that our notion of \emph{adversarial degree} is different from the standard notion of degree used in graph theory, and in particular different from the degree notion in the classical one-inclusion graph used above. Specifically, we do \emph{not} count all edges incident on a vertex rather we count the number of datapoints $(x,y)$ in a vertex that witness an edge. This different notion of degree is more suitable for our purposes and is related to the average leave-one-out \emph{robust} error.

\subsection{From orientations to learners}

An orientation $\calO: E_n\to V_n$ of the global one-inclusion graph $G_\calH^\calU=(V_n,E_n)$ is a mapping that directs each edge $e=(\SET{u,v},z)\in E_n$ towards a vertex $\calO(e)\in \SET{u,v}$. Given an orientation $\calO:E_n\to V_n$ of the global one-inclusion graph $G^\calU_\calH$, the adversarial out-degree of a vertex $v\in V$, denoted by $\outdeg(v; \calO)$, is defined as the number of elements $(x,y)\in v$ that witness an out-going edge incident on $v$ according to orientation $\calO$:
\begin{equation}
\label{eqn:outdeg-v2}
\outdeg(v;\calO) = \left|\left\{ (x,y)\in v \middle|
  \begin{array}{l} \exists u \in V_n,z\in \calX \text{ s.t. }(\SET{v, u},z)\in E_n~\wedge\\
  (x,y)\in v\Delta u\wedge \calO((\SET{v,u},z))=u
  \end{array}\right\}\right|.
\end{equation}

Why are we interested in orientations of the \graph~$G_\calH^\calU$? We show next that every orientation of $G_\calH^\calU$ can be used to construct a learner, and that the expected robust risk of this learner is bounded from above by the maximum adversarial out-degree of the corresponding orientation. We will use this observation later in \prettyref{sec:optimal} to construct an optimal learner.

\begin{lem}
\label{lem:orientation-to-prediction}
For any class $\calH$, any perturbation set $\calU$, and any $n > 1$, let $G_\calH^\calU=(V_n,E_n)$ be the global one-inclusion graph. Then, for any orientation $\calO:E_n\to V_n$ of $G_\calH^\calU$, there exists a learner $\bbA_\calO:(\calX\times\calY)^{n-1}\to \calY^\calX$, such that 
\[\calE_{n-1}(\bbA_\calO;\calH,\calU) \leq \frac{\max_{v\in V_n} \outdeg(v;\calO)}{n}.\]
\end{lem}
The proof is deferred to \prettyref{app:orientation-to-learner}. At a high-level, we can use an orientation $\calO$ of $G_\calH^\calU$ to make predictions in the following way: upon receiving training examples $S$ and a (possibly adversarial) test instance $z$, we consider all possible natural datapoints $(x,y)$ of which $z$ is a perturbation of $x$ (i.e., $z\in\calU(x)$) such that $S\cup\SET{(x,y)}$ can be labeled robustly using class $\calH$ \wrt $\calU$ (note that these are all vertices in $G^\calU_{\calH}$ by definition), and if two different robust labelings of $z$ are possible, the orientation $\calO$ determines which label to predict. This is defined explicitly in \prettyref{alg:orientation-to-learner}.
\begin{algorithm2e}[H]
\caption{Converting an Orientation $\calO$ of $G_\calH^\calU$ to a Learner $\bbA_\calO$.}
\label{alg:orientation-to-learner}
\SetKwInput{KwInput}{Input}                
\SetKwInput{KwOutput}{Output}              
\SetKwFunction{FMain}{CycleRobust}
\SetKwFunction{FDisc}{Discretizer}
\DontPrintSemicolon
    \KwInput{Training dataset $S=\SET{(x_1,y_1),\dots, (x_{n-1},y_{n-1})} \in (\calX\times \calY)^{n-1}$, test instance $z\in \calX$, and an orientation $\calO:E_n\to V_n$ of $G_\calH^\calU=(E_n, V_n)$.}
  \BlankLine
  Let $P_+= \SET{v \in V_n: \exists x \in \calX \text{ s.t. } z\in \calU(x) \wedge v=\SET{(x_1,y_1),\dots,(x_{n-1},y_{n-1}),(x,+1)}}$.\;
  Let $P_-=\SET{v \in V_n: \exists x \in \calX \text{ s.t. } z\in \calU(x) \wedge v=\SET{(x_1,y_1),\dots,(x_{n-1},y_{n-1}),(x,-1)}}$.\;
  \textbf{If $\inparen{\exists_{y\in \SET{\pm1}}}\inparen{\exists_{v\in P_{y}}}\inparen{\forall_{u\in P_{-y}}}: \calO((\SET{v,u}, z))=v$, then} output label $y$.\;
  \textbf{Otherwise}, output $+1$.
\end{algorithm2e}

\section{A generic minimax optimal learner}
\label{sec:optimal}
We now present an optimal robust learner based on our \graph~from \prettyref{sec:graph}.
\begin{tcolorbox}[size=small,colback=white,colframe=black]
For any class $\calH$, any perturbation set $\calU$, and integer $n>1$, let $G_\calH^\calU=(V_n,E_n)$ be the global one-inclusion graph (Equations~\ref{eqn:oneinclusion-vertices-v2} and \ref{eqn:oneinclusion-edges-v2}). Let $\calO^*$ be an orientation that minimizes the maximum adversarial out-degree of $G_\calH^{\calU}$:
\begin{equation}
    \label{eqn:optimal-orientation}
    \calO^* \in \argmin_{\calO:E_n\to V_n}~~~\max_{v\in V_n}~\outdeg(v;\calO).
\end{equation}
Then, let $\bbG_{\calH,\calU}$ be the learner implied by orientation $\calO^*$ as described in \prettyref{alg:orientation-to-learner}.
\end{tcolorbox}
\begin{thm}
\label{thm:optimal}
For any $\calH$, $\calU$, any $n\in \bbN$, learner $\bbG_{\calH,\calU}$ described above satisfies for any learner $\bbA$:
$$\calE_{2n-1}(\bbG_{\calH,\calU};\calH,\calU) \leq 4\cdot \calE_{n}(\bbA;\calH,\calU), \text{ \& equivalently }\calM^{\re}_{\eps}(\bbG_{\calH,\calU};\calH,\calU)\leq 2\cdot \calM^{\re}_{\eps/4}(\bbA;\calH,\calU)-1.$$
\end{thm}

Before proceeding to the proof of \prettyref{thm:optimal}, we first prove a key Lemma which basically shows that we can use an arbitrary learner $\bbA$ to orient the edges in the global one-inclusion graph $G_\calH^\calU$, and that the maximum adversarial out-degree of the resultant orientation is upper bounded by the robust error rate of $\bbA$. 
\begin{lem} [Lowerbound on Error Rate of Learners]
\label{lem:lowerbound-inductive}
Let $\bbA:(\calX\times \calY)^*\to \calY^\calX$ be any learner, and $n\in \bbN$. Let $G_\calH^\calU=(V_{2n},E_{2n})$ be the global one-inclusion graph as defined in \prettyref{eqn:oneinclusion-vertices-v2} and \prettyref{eqn:oneinclusion-edges-v2}. Then, there exists an orientation $\calO_\bbA:E_{2n}\to V_{2n}$ of $G^\calU_{\calH}$ such that $$\calE_n(\bbA;\calH,\calU)\geq \frac{1}{4}\frac{\max_{v\in V_{2n}} \outdeg(v;\calO_\bbA)}{2n}.$$
\end{lem}
\begin{proof}
We begin with describing the orientation $\calO_\bbA$ by orienting edges incident on each vertex $v\in V_{2n}$. Consider an arbitrary vertex $v=\SET{(x_1,y_1),\dots, (x_{2n},y_{2n})}$ and let $P_v$ be a uniform distribution over $(x_1,y_1),\dots, (x_{2n},y_{2n})$. For each $1\leq t \leq 2n$, let $$p_t(v)=\Prob_{S\sim P_v^n} \insquare{\exists z\in \calU(x_t): \bbA(S)(z)\neq y_t | (x_t,y_t)\notin S}.$$
For each $(x_t,y_t)\in v$ that witnesses an edge, i.e. $\exists u \in V_{2n}, z\in \calX\text{ s.t. }(\SET{v, u}, z)\in E_{2n}$ and $(x_t,y_t)\in v\Delta u$, if $p_t < \frac{1}{2}$, then orient \emph{all} edges incident on $(x_t,y_t)$ inward, otherwise orient them arbitrarily. Note that this might yield edges that are oriented outwards from both their endpoint vertices, in which case, we arbitrarily orient such an edge. Observe also that we will not encounter a situation where edges are oriented inwards towards both their endpoints (which is an invalid orientation). This is because for any two vertices $v,u\in V_{2n}$ such that $\exists z_0\in \calX$ where $(\SET{u,v},z_0)\in E_{2n}$ and $v \Delta u = \SET{(x_t,y_t),(\Tilde{x}_t,-y_t)}$, we cannot have $p_{t}(v) < \frac{1}{2}$ and $p_t(u) < \frac{1}{2}$,
since
\[p_t(v) \geq \Prob_{S\sim P_v^m}\insquare{\bbA(S)(z_0)\neq y_t | (x_t,y_t)\notin S}~~\text{and}~~p_t(u) \geq \Prob_{S\sim P_u^m}\insquare{\bbA(S)(z_0)\neq -y_t | (\Tilde{x}_t,-y_t)\notin S},\]
and $P_v$ conditioned on $(x_t,y_t)\notin S$ is the same distribution as $P_u$ conditioned on $(\Tilde{x}_t,-y_t)\notin S$. This concludes describing the orientation $\calO_\bbA$. We now bound the adversarial out-degree of vertices $v\in V_{2n}$:
\begin{align*}
    &\outdeg(v;\calO_\bbA)\leq \sum_{t=1}^{2n}\ind\insquare{p_t\geq \frac{1}{2}} \leq 2\sum_{t=1}^{2n} p_t=2\sum_{t=1}^{2n} \Prob_{S\sim P^n} \insquare{\exists z\in \calU(x_t): \bbA(S)(z)\neq y_t | (x_t,y_t)\notin S}\\
    &=2 \sum_{t=1}^{2n} \frac{\Prob_{S\sim P^n} \insquare{\inparen{\exists z\in \calU(x_t): \bbA(S)(z)\neq y_t} \wedge (x_t,y_t)\notin S}}{\Prob_{S\sim P^n}\insquare{(x_t,y_t)\notin S}}\\
    &\leq 4 \sum_{t=1}^{2n} \Prob_{S\sim P^n} \insquare{\inparen{\exists z\in \calU(x_t): \bbA(S)(z)\neq y_t} \wedge (x_t,y_t)\notin S}
    = 4 \Ex_{S\sim P^n} \sum_{(x_t,y_t)\notin S} \ind\insquare{\exists z\in \calU(x_t): \bbA(S)(z)\neq y_t}\\ &\leq 4 \Ex_{S\sim P^n} \sum_{t=1}^{2n} \ind\insquare{\exists z\in \calU(x_t): \bbA(S)(z)\neq y_t} = 8n\Ex_{S\sim P^n} \Risk_\calU(\bbA(S);P) \leq 8n \calE_n(\bbA;\calH,\calU).
\end{align*}
Since the above holds for any vertex $v\in V_{2n}$, by rearranging terms, we get $\calE_n(\bbA;\calH,\calU)\geq \frac{1}{4}\frac{\max_{v\in V_{2n}} \outdeg(v;\calO_\bbA)}{2n}$.
\end{proof}

We are now ready to proceed with the proof of \prettyref{thm:optimal}.
\begin{proof}[of \prettyref{thm:optimal}]
By invoking \prettyref{lem:lowerbound-inductive}, we have that for any learner $\bbA$, 
\[\calE_n(\bbA;\calH,\calU)\geq \frac{1}{4}\frac{\max_{v\in V_{2n}} \outdeg(v;\calO_\bbA)}{2n}.\]
By \prettyref{eqn:optimal-orientation}, an orientation $\calO^{*}$ has smaller maximum adversarial out-degree, thus
\[\frac{1}{4}\frac{\max_{v\in V_{2n}} \outdeg(v;\calO_\bbA)}{2n} \geq \frac{1}{4}\frac{\max_{v\in V_{2n}} \outdeg(v;\calO^*)}{2n}.\]
By invoking \prettyref{lem:orientation-to-prediction}, it follows that our optimal learner $\bbG_{\calH,\calU}$ satisfies
        \[\frac{1}{4}\frac{\max_{v\in V_{2n}} \outdeg(v;\calO^*)}{2n} \geq \frac{\calE_{2n-1}(\bbG_{\calH,\calU};\calH,\calU)}{4}.\]
We arrive at the theorem statement by chaining the above inequalities and rearranging terms.
\end{proof}

\section{A complexity measure and sample complexity bounds}
\label{sec:complexity-measure}

In \prettyref{sec:optimal}, we showed how our \graph~yields a near-optimal learner for adversarially robust learning. We now turn to characterizing adversarially robust learnability. 

Across learning theory, many fundamental learning problems can be surprisingly characterized by means of combinatorial complexity measures. Such characterizations are often quantitatively insightful in that they provide tight bounds on the number of examples needed for learning, and also insightful for algorithm design. For example, for standard (non-robust) learning, the VC dimension characterizes what classes $\calH$ are PAC learnable \citep*{vapnik:71,vapnik:74,blumer:89,ehrenfeucht:89}. For multi-class learning, there are characterizations based on the Natarjan and Graph dimensions, and the Daniely-Shalev-Shwartz (DS) \citep*{natarajan:89, daniely:15, DBLP:journals/corr/abs-2203-01550}. For learning real-valued functions, the fat-shattering dimension plays a similar role \citep*{DBLP:journals/jacm/AlonBCH97, DBLP:journals/jcss/KearnsS94,DBLP:journals/siamcomp/Simon97}. The Littlestone dimension characterizes online learnability \citep*{DBLP:journals/ml/Littlestone87}, and the star number characterizes the label complexity of active learning \citep*{hanneke:15b}. 

\citet*{DBLP:journals/jmlr/Shalev-ShwartzSSS10} showed that in Vapnik's ``General Learning'' problem \citep{vapnik:82}, the loss class having finite VC dimension is sufficient but not, in general, necessary for learnability and asked whether there is another dimension that characterizes learnability in this setting. But recently, \citet*{DBLP:journals/natmi/Ben-DavidHMSY19} surprisingly exhibited a statistical learning problem that can not be characterized with a combinatorial VC-like dimension. In order to do so, they presented a \emph{formal} definition of the notion of ``dimension'' or ``complexity measure'' (see \prettyref{def:complexity-measure}), that all previously proposed dimensions in statistical learning theory comply with. This raises the following natural question:
\begin{center}
    {\em Is there a dimension that characterizes robust learnability, and if so, what is it?!}
\end{center}

\subsection{A dimension characterizing robust learning}
\label{sec:dimension-charac}
We present next a dimension for adversarially robust learnability, which is inspired by our \graph~described in \prettyref{sec:graph}. 
\begin{equation}
    \label{eqn:complexity-measure}
     \CM = \max\left\{ n\in \bbN \cup \SET{\infty} \middle|
  \begin{array}{l} \exists \text{ a finite subgraph }G=(V,E)\text{ of }G^{\calU}_{\calH}=(V_n,E_n) \text{ s.t. }\\
  \forall\text{ orientations }\calO\text{ of }G, \exists v\in V\text{ where }\outdeg(v;\calO)\geq \frac{n}{3}.
  \end{array}\right\}.
\end{equation}

In \prettyref{app:finite-char-property}, we discuss how our dimension satisfies the formal definition proposed by \citep{DBLP:journals/natmi/Ben-DavidHMSY19}. We now show that $\CM$ characterizes robust learnability \emph{qualitatively} and \emph{quantitatively}.

\begin{thm}[Qualitative Characterization]
\label{thm:realizable-qual}
For any class $\calH$ and any perturbation set $\calU$, $\calH$ is robustly PAC learnable with respect to $\calU$ \emph{if and only if} $\CM$ is finite.
\end{thm}

\begin{thm} [Quantitative Characterization]
\label{thm:realizable-bounds}
For any class $\calH$ and any perturbation set $\calU$,
\[ \Omega\inparen{\frac{\CM}{\eps}}\leq \calM^\re_{\eps,\delta}(\calH,\calU) \leq O\inparen{\frac{\CM}{\eps}\log^2 \frac{\CM}{\eps} + \frac{\log(1/\delta)}{\eps}}.\]
\end{thm}

\prettyref{thm:realizable-qual} follows immediately from \prettyref{thm:realizable-bounds}. To prove \prettyref{thm:realizable-bounds}, we first prove the following key Lemma which provides upper and lower bounds on the \emph{minimax} expected robust risk of learning a class $\calH$ \wrt a perturbation set $\calU$ (see \prettyref{eqn:exp-rob-risk}) as a function of our introduced dimension $\CM$. \prettyref{thm:realizable-bounds} follows from an argument to boost the robust risk and the confidence as appeared in \citet*{pmlr-v99-montasser19a}. The proofs are deferred to \prettyref{app:realizable-bounds}.

\begin{lem}
\label{lem:expectation-bounds}
For any class $\calH$, any perturbation set $\calU$, and any $\eps\in(0,1)$,
\begin{enumerate}
    \item $\forall n > \CM: \calE_{n-1}(\calH,\calU)\leq \frac{1}{3}$.
    \item $\forall 2 \leq n \leq \frac{\CM}{2}: \calE_{\frac{n}{\eps}}(\calH,\calU)\geq\frac{\eps}{6}$.
\end{enumerate}
\end{lem}

\subsection{Examples}

We discuss a few ways of estimating or calculating our proposed dimension $\CM$.

\begin{prop}
\label{prop:dimension-upper-bounds}
For any class $\calH$ and perturbation set $\calU$:$$\CM \leq \min\SET{\Tilde{O}(\vc(\calH)\vc^*(\calH)), \Tilde{O}(\vc(\calL^\calU_\calH))},$$
where $\vc^*(\calH)$ denotes the \emph{dual} VC dimension, and $\vc(\calL^\calU_\calH))$ denotes the VC dimension of the robust-loss class $\calL_\calH^\calU = \SET{(x,y)\mapsto \sup_{z\in \calU(x)}\ind[h(z\neq y)]:h\in\calH}$.
\end{prop}
\begin{proof}
Set $\eps_0=\frac{1}{3}$. We know from \prettyref{thm:realizable-bounds} that $\calM^\re_{\eps_0}(\calH,\calU)\geq \Omega(\CM)$. We also know from \citep*[Theorem 4 in][]{pmlr-v99-montasser19a} that $\calM^\re_{\eps_0}(\calH,\calU)\leq \Tilde{O}(\vc(\calH)\vc^*(\calH))$. Finally, we know from \citep*[Theorem 1 in][]{DBLP:conf/nips/CullinaBM18} that $\calM^\re_{\eps_0}(\calH,\calU)\leq \Tilde{O}(\vc(\calL^\calU_\calH))$. Combining these together yields that stated bound.
\end{proof}

The dual VC dimension satisfies: $\vc^{*}(\calH) < 2^{\vc(\calH)+1}$ \citep{assouad:83}, and this exponential dependence is tight for some classes. For many natural classes, however, such as linear predictors and some neural networks \citep[see Lemma 3.2 in ][]{DBLP:conf/nips/MontasserHS20}, the primal and dual VC dimensions are equal, or at least polynomially related. Using \prettyref{prop:dimension-upper-bounds}, we can conclude that for such classes $\CM\leq {\rm poly}(\vc(\calH))$, specifically, for $\calH$ being linear predictors, $\CM \leq \Tilde{O}(\vc^2(\calH))$. Furthermore, for $\calH$ being linear predictors and $\calU=\ell_p$ perturbations, we know that $\vc(\calL_\calH^\calU)=O(\vc(\calH))$ \citep*[Theorem 2 in ][]{DBLP:conf/nips/CullinaBM18}, and so using \prettyref{prop:dimension-upper-bounds} again, we get a tighter bound for these $\ell_p$ perturbations $\CM\leq \Tilde{O}(\vc(\calH))$. 

While \prettyref{prop:dimension-upper-bounds} is certainly useful for estimating our dimension $\CM$, we get vacuous bounds when the VC dimension $\vc(\calH)$ is infinite. To this end, recall the $(\calH,\calU)$ examples in \prettyref{exam:all-labels} and \prettyref{exam:halfspaces} mentioned in \prettyref{sec:intro}, which satisfy $\vc(\calH)=\infty$. We can calculate $\CM$ for these examples differently. In particular, in \prettyref{exam:all-labels}, by definition, the \graph~$G^\calU_\calH=(V_n,E_n)$ has no edges when $n>1$ because $\calU(x)=\calX$ and thus $\CM\leq 1$. In \prettyref{exam:halfspaces}, we get that $\CM \leq \Tilde{O}(d)$ since we can robustly learn with $O(d)$ samples, but we can also calculate $\CM$ directly by constructing the \graph~$G_\calH^\calU = (V_n, E_n)$ for $n>3d$ and observing that we can orient $G_\calH^\calU$ such that the adversarial out-degree is at most $d$, which is possible because of the definition of $\calU$.

\subsection{Conjectures}
\label{sec:conjectures}
While we have shown that our proposed dimension $\CM$ in \prettyref{eqn:complexity-measure} characterizes robust learnability, we believe that there are other {\em equivalent} dimensions that are \emph{simpler} to describe. For a more appealing dimension, we may take inspiration from recent progress on multi-class learning \citep{DBLP:journals/corr/abs-2203-01550}, where it was shown that the DS dimension due to \citet{DBLP:conf/colt/DanielyS14} characterizes multi-class learnability. For a class $\calH\subseteq \calY^\calX$ ($\abs{\calY}>2$), the DS dimension corresponds to the largest $n$ s.t.~there exists points $P\in\calX^n$ where the projection of $\calH$ onto $P$ induces a one-inclusion hyper graph where every vertex has full-degree. This inspires the \emph{full-degree} dimension of the \graph:
\[ {\mathfrak{FD}}_{\calU}(\calH)=\max\left\{ n\in \bbN \cup \SET{\infty} \middle|
  \begin{array}{l} \exists \text{ a finite subgraph }G=(V,E)\text{ of }G^{\calU}_{\calH}=(V_n,E_n) \text{ s.t. every}\\
  \text{vertex has full-degree: }\forall v\in V, \adeg(v;E)\geq n.
  \end{array}\right\}.\]
This complexity measure avoids orientations, and thus, it is perhaps simpler to verify $``\mathfrak{FD}_{\calU}(\calH)\geq d"$ than $``\mathfrak{D}_{\calU}(\calH)\geq d"$. Furthermore, when $\calU(x)=\SET{x}\forall x\in \calX$, the full-degree dimension, $\mathfrak{FD}_{\calU}(\calH)$, corresponds exactly to the VC dimension of $\calH$, $\vc(\calH)$.
\begin{conj}
\label{conj:fulldegree-measure}
For any class $\calH$ and perturbation set $\calU$, $\calM^\re_{\eps,\delta}(\calH,\calU) = \Theta_{\eps,\delta}\inparen{{\mathfrak{FD}}_{\calU}(\calH)}$.
\end{conj}

\citet*{pmlr-v99-montasser19a} proposed the following combinatorial robust shattering dimension, denoted $\dim_\calU(\calH)$, and showed that $\calM^\re_{\eps,\delta}(\calH,\calU)\geq \Omega_{\eps,\delta}(\dim_\calU(\calH))$. 
\begin{defn}[Robust Shattering Dimension]
\label{def:robustshatter-dim-orig}
A sequence $z_1,\ldots,z_k \in \calX$ is said to be \emph{$\calU$-robustly shattered} by $\calH$ if $\exists x^+_1,x^-_1,\dots,x^{+}_k,x^{-}_k\in\calX$ s.t. $\forall i\in[k], z_i\in\calU(x^+_i)\cap\calU(x^-_i)$ and $\forall y_1,\dots,y_k \in \SET{\pm 1}:\exists h\in \calH$ such that $h(z')=y_i \forall z'\in \calU(x^{y_i}_i) \forall 1\leq i\leq k$. The \emph{$\calU$-robust shattering dimension} $\dim_{\calU}(\calH)$ is defined as the largest $k$ for which there exist $k$ points $\calU$-robustly shattered by $\calH$.
\end{defn}
In regards to the relationship between the robust shattering dimension $\dim_\calU(\calH)$ and our dimension $\CM$, we conjecture that our dimension can be arbitrarily larger. In other words, we conjecture that the robust shattering dimension $\dim_\calU(\calH)$ does not characterize robust learnability.

\begin{conj}
\label{conj:robust-shattering}
$\forall n\in \bbN, \exists \calX, \calH,\calU$, such that $\dim_{\calU}(\calH)=O(1)$ but $\CM\geq n$.
\end{conj}
We find this to be analogous to a separation in multi-class learnability, where the Natarajan dimension was shown to not characterize multi-class learnability \citep{DBLP:journals/corr/abs-2203-01550}. Because in both graphs, the one-inclusion hyper graph and our \graph, the Natarajan and robust shattering dimensions represent a ``cube'' in their corresponding graph, while the DS dimension and our full-degree dimension represent a ``pseudo-cube'' in the terminology of \citet{DBLP:journals/corr/abs-2203-01550}.

Another interesting and perhaps useful direction to explore is the relationship between our proposed complexity measure $\CM$ and the VC dimension. We believe that it is actually possible to orient the \graph~such that the maximum adversarial out-degree is at most $O(\vc(\calH))$. 

\begin{conj}
\label{conj:vc-upper-bound}
For any class $\calH$ and perturbation set $\calU$, $\CM \leq O(\vc(\calH))$.
\end{conj}

\section{Agnostic robust learnability}
\label{sec:agnostic}

For the agnostic setting, we consider robust learnability \wrt arbitrary distributions $\calD$ that are {\em not} necessarily robustly realizable, i.e., $\calD\notin {\rm RE}(\calH,\calU)$ (see \prettyref{def:ag_sample_complexity} in \prettyref{app:section-5}). We can establish an upper bound in the agnostic setting via reduction to the realizable case, following an argument from \citet*{david:16} and later applied to agnostic robust learning by \citet*{pmlr-v99-montasser19a}:
\begin{thm}
\label{thm:agnostic}
For any class $\calH$ and any perturbation set $\calU$,
\[\calM^{\ag}_{\eps,\delta}(\calH,\calU) = O\inparen{\frac{\CM}{\eps^2}\log^2\inparen{\frac{\CM}{\eps}}+\frac{1}{\eps^2}\log\inparen{\frac{1}{\delta}}}.\]
\end{thm}

This is achieved by applying the agnostic-to-realizable reduction to the optimal learner $\bbG_{\calH,\calU}$ that we get from orienting the graph $G_{\calH}^{\calU}=(V_{\CM+1},E_{\CM+1})$.
The reduction is stated abstractly in the following Lemma whose proof is provided in \prettyref{app:section-5}.
\begin{lem}
\label{lem:reduction-to-realizable}
For any well-defined realizable learner $\bbA$, there is an agnostic learner $\bbB$ such that
\[ \calM^{\re}_{\eps}(\bbA;\calH,\calU) \leq \calM^{\ag}_{\eps,\delta}(\bbB;\calH,\calU) \leq O\inparen{\frac{\calM^{\re}_{1/3}(\bbA;\calH,\calU)}{\eps^2}\log^2\inparen{\frac{\calM^{\re}_{1/3}(\bbA;\calH,\calU)}{\eps}} + \frac{1}{\eps^2}\log\inparen{\frac{1}{\delta}}}.\]
\end{lem}
\prettyref{thm:agnostic} immediately follows by combining \prettyref{lem:reduction-to-realizable} and \prettyref{thm:realizable-bounds}.

\section*{Acknowledgments}

We thank Shay Moran for numerous enlightening discussions about 
structures arising in the analysis of muliclass learning, concurrent 
to this work, which seem tantalizingly related to those arising here. This work is supported in part by DARPA under cooperative agreement HR00112020003 \footnote{The views expressed in this work do not necessarily reflect the position or the policy of the Government and no official endorsement should be inferred. Approved for public release; distribution is unlimited.}, and in part by the NSF/Simons sponsored Collaboration on the Theoretical Foundations of Deep Learning (\url{https://deepfoundations.ai}).

\bibliography{learning}

\begin{thebibliography}{35}
\providecommand{\natexlab}[1]{#1}
\providecommand{\url}[1]{\texttt{#1}}
\expandafter\ifx\csname urlstyle\endcsname\relax
  \providecommand{\doi}[1]{doi: #1}\else
  \providecommand{\doi}{doi: \begingroup \urlstyle{rm}\Url}\fi

\bibitem[Alon et~al.(1997)Alon, Ben{-}David, Cesa{-}Bianchi, and
  Haussler]{DBLP:journals/jacm/AlonBCH97}
Noga Alon, Shai Ben{-}David, Nicol{\`{o}} Cesa{-}Bianchi, and David Haussler.
\newblock Scale-sensitive dimensions, uniform convergence, and learnability.
\newblock \emph{J. {ACM}}, 44\penalty0 (4):\penalty0 615--631, 1997.
\newblock \doi{10.1145/263867.263927}.
\newblock URL \url{https://doi.org/10.1145/263867.263927}.

\bibitem[Alon et~al.(2021)Alon, Hanneke, Holzman, and
  Moran]{DBLP:conf/focs/AlonHHM21}
Noga Alon, Steve Hanneke, Ron Holzman, and Shay Moran.
\newblock A theory of {PAC} learnability of partial concept classes.
\newblock In \emph{62nd {IEEE} Annual Symposium on Foundations of Computer
  Science, {FOCS} 2021, Denver, CO, USA, February 7-10, 2022}, pages 658--671.
  {IEEE}, 2021.
\newblock \doi{10.1109/FOCS52979.2021.00070}.
\newblock URL \url{https://doi.org/10.1109/FOCS52979.2021.00070}.

\bibitem[Assouad(1983)]{assouad:83}
P.~Assouad.
\newblock Densit\'e et dimension.
\newblock \emph{Annales de l'Institut Fourier (Grenoble)}, 33\penalty0
  (3):\penalty0 233--282, 1983.

\bibitem[Ben{-}David et~al.(2019)Ben{-}David, Hrubes, Moran, Shpilka, and
  Yehudayoff]{DBLP:journals/natmi/Ben-DavidHMSY19}
Shai Ben{-}David, Pavel Hrubes, Shay Moran, Amir Shpilka, and Amir Yehudayoff.
\newblock Learnability can be undecidable.
\newblock \emph{Nat. Mach. Intell.}, 1\penalty0 (1):\penalty0 44--48, 2019.
\newblock \doi{10.1038/s42256-018-0002-3}.
\newblock URL \url{https://doi.org/10.1038/s42256-018-0002-3}.

\bibitem[Biggio et~al.(2013)Biggio, Corona, Maiorca, Nelson, {\v{S}}rndi{\'c},
  Laskov, Giacinto, and Roli]{biggio2013evasion}
Battista Biggio, Igino Corona, Davide Maiorca, Blaine Nelson, Nedim
  {\v{S}}rndi{\'c}, Pavel Laskov, Giorgio Giacinto, and Fabio Roli.
\newblock Evasion attacks against machine learning at test time.
\newblock In \emph{Joint European conference on machine learning and knowledge
  discovery in databases}, pages 387--402. Springer, 2013.

\bibitem[Blumer et~al.(1989)Blumer, Ehrenfeucht, Haussler, and
  Warmuth]{blumer:89}
A.~Blumer, A.~Ehrenfeucht, D.~Haussler, and M.~Warmuth.
\newblock Learnability and the {Vapnik-Chervonenkis} dimension.
\newblock \emph{Journal of the Association for Computing Machinery},
  36\penalty0 (4):\penalty0 929--965, 1989.

\bibitem[Brukhim et~al.(2022)Brukhim, Carmon, Dinur, Moran, and
  Yehudayoff]{DBLP:journals/corr/abs-2203-01550}
Nataly Brukhim, Daniel Carmon, Irit Dinur, Shay Moran, and Amir Yehudayoff.
\newblock A characterization of multiclass learnability.
\newblock \emph{CoRR}, abs/2203.01550, 2022.
\newblock \doi{10.48550/arXiv.2203.01550}.
\newblock URL \url{https://doi.org/10.48550/arXiv.2203.01550}.

\bibitem[Cohen et~al.(2019)Cohen, Rosenfeld, and
  Kolter]{DBLP:conf/icml/CohenRK19}
Jeremy~M. Cohen, Elan Rosenfeld, and J.~Zico Kolter.
\newblock Certified adversarial robustness via randomized smoothing.
\newblock In Kamalika Chaudhuri and Ruslan Salakhutdinov, editors,
  \emph{Proceedings of the 36th International Conference on Machine Learning,
  {ICML} 2019, 9-15 June 2019, Long Beach, California, {USA}}, volume~97 of
  \emph{Proceedings of Machine Learning Research}, pages 1310--1320. {PMLR},
  2019.
\newblock URL \url{http://proceedings.mlr.press/v97/cohen19c.html}.

\bibitem[Cullina et~al.(2018)Cullina, Bhagoji, and
  Mittal]{DBLP:conf/nips/CullinaBM18}
Daniel Cullina, Arjun~Nitin Bhagoji, and Prateek Mittal.
\newblock Pac-learning in the presence of adversaries.
\newblock In \emph{Advances in Neural Information Processing Systems 31: Annual
  Conference on Neural Information Processing Systems 2018, NeurIPS 2018, 3-8
  December 2018, Montr{\'{e}}al, Canada}, pages 228--239, 2018.
\newblock URL
  \url{http://papers.nips.cc/paper/7307-pac-learning-in-the-presence-of-adversaries}.

\bibitem[Daniely et~al.(2015)Daniely, Sabato, Ben-David, and
  Shalev-{S}hwartz]{daniely:15}
A.~Daniely, S.~Sabato, S.~Ben-David, and S.~Shalev-{S}hwartz.
\newblock Multiclass learnability and the {ERM} principle.
\newblock \emph{Journal of Machine Learning Research}, 16:\penalty0 2377--2404,
  2015.

\bibitem[Daniely and Shalev{-}Shwartz(2014)]{DBLP:conf/colt/DanielyS14}
Amit Daniely and Shai Shalev{-}Shwartz.
\newblock Optimal learners for multiclass problems.
\newblock In Maria{-}Florina Balcan, Vitaly Feldman, and Csaba
  Szepesv{\'{a}}ri, editors, \emph{Proceedings of The 27th Conference on
  Learning Theory, {COLT} 2014, Barcelona, Spain, June 13-15, 2014}, volume~35
  of \emph{{JMLR} Workshop and Conference Proceedings}, pages 287--316.
  JMLR.org, 2014.
\newblock URL \url{http://proceedings.mlr.press/v35/daniely14b.html}.

\bibitem[David et~al.(2016)David, Moran, and Yehudayoff]{david:16}
O.~David, S.~Moran, and A.~Yehudayoff.
\newblock Supervised learning through the lens of compression.
\newblock In \emph{Advances in Neural Information Processing Systems 29}, pages
  2784--2792, 2016.

\bibitem[Ehrenfeucht et~al.(1989)Ehrenfeucht, Haussler, Kearns, and
  Valiant]{ehrenfeucht:89}
A.~Ehrenfeucht, D.~Haussler, M.~Kearns, and L.~Valiant.
\newblock A general lower bound on the number of examples needed for learning.
\newblock \emph{Information and Computation}, 82\penalty0 (3):\penalty0
  247--261, 1989.

\bibitem[Goodfellow et~al.(2015)Goodfellow, Shlens, and
  Szegedy]{DBLP:journals/corr/GoodfellowSS14}
Ian~J. Goodfellow, Jonathon Shlens, and Christian Szegedy.
\newblock Explaining and harnessing adversarial examples.
\newblock In Yoshua Bengio and Yann LeCun, editors, \emph{3rd International
  Conference on Learning Representations, {ICLR} 2015, San Diego, CA, USA, May
  7-9, 2015, Conference Track Proceedings}, 2015.
\newblock URL \url{http://arxiv.org/abs/1412.6572}.

\bibitem[Graepel et~al.(2005)Graepel, Herbrich, and Shawe-{T}aylor]{graepel:05}
T.~Graepel, R.~Herbrich, and J.~Shawe-{T}aylor.
\newblock {PAC}-{B}ayesian compression bounds on the prediction error of
  learning algorithms for classification.
\newblock \emph{Machine Learning}, 59\penalty0 (1-2):\penalty0 55--76, 2005.

\bibitem[Hanneke and Yang(2015)]{hanneke:15b}
S.~Hanneke and L.~Yang.
\newblock Minimax analysis of active learning.
\newblock \emph{Journal of Machine Learning Research}, 16\penalty0
  (12):\penalty0 3487--3602, 2015.

\bibitem[Haussler et~al.(1994)Haussler, Littlestone, and
  Warmuth]{DBLP:journals/iandc/HausslerLW94}
David Haussler, Nick Littlestone, and Manfred~K. Warmuth.
\newblock Predicting {\textbackslash}0,1{\textbackslash}-functions on randomly
  drawn points.
\newblock \emph{Inf. Comput.}, 115\penalty0 (2):\penalty0 248--292, 1994.
\newblock \doi{10.1006/inco.1994.1097}.
\newblock URL \url{https://doi.org/10.1006/inco.1994.1097}.

\bibitem[Kearns and Schapire(1994)]{DBLP:journals/jcss/KearnsS94}
Michael~J. Kearns and Robert~E. Schapire.
\newblock Efficient distribution-free learning of probabilistic concepts.
\newblock \emph{J. Comput. Syst. Sci.}, 48\penalty0 (3):\penalty0 464--497,
  1994.
\newblock \doi{10.1016/S0022-0000(05)80062-5}.
\newblock URL \url{https://doi.org/10.1016/S0022-0000(05)80062-5}.

\bibitem[Littlestone(1987)]{DBLP:journals/ml/Littlestone87}
Nick Littlestone.
\newblock Learning quickly when irrelevant attributes abound: {A} new
  linear-threshold algorithm.
\newblock \emph{Mach. Learn.}, 2\penalty0 (4):\penalty0 285--318, 1987.
\newblock \doi{10.1007/BF00116827}.
\newblock URL \url{https://doi.org/10.1007/BF00116827}.

\bibitem[Madry et~al.(2018)Madry, Makelov, Schmidt, Tsipras, and
  Vladu]{DBLP:conf/iclr/MadryMSTV18}
Aleksander Madry, Aleksandar Makelov, Ludwig Schmidt, Dimitris Tsipras, and
  Adrian Vladu.
\newblock Towards deep learning models resistant to adversarial attacks.
\newblock In \emph{6th International Conference on Learning Representations,
  {ICLR} 2018, Vancouver, BC, Canada, April 30 - May 3, 2018, Conference Track
  Proceedings}. OpenReview.net, 2018.
\newblock URL \url{https://openreview.net/forum?id=rJzIBfZAb}.

\bibitem[Montasser et~al.(2019)Montasser, Hanneke, and
  Srebro]{pmlr-v99-montasser19a}
Omar Montasser, Steve Hanneke, and Nathan Srebro.
\newblock Vc classes are adversarially robustly learnable, but only improperly.
\newblock In Alina Beygelzimer and Daniel Hsu, editors, \emph{Proceedings of
  the Thirty-Second Conference on Learning Theory}, volume~99 of
  \emph{Proceedings of Machine Learning Research}, pages 2512--2530, Phoenix,
  USA, 25--28 Jun 2019. PMLR.

\bibitem[Montasser et~al.(2020)Montasser, Hanneke, and
  Srebro]{DBLP:conf/nips/MontasserHS20}
Omar Montasser, Steve Hanneke, and Nati Srebro.
\newblock Reducing adversarially robust learning to non-robust {PAC} learning.
\newblock In Hugo Larochelle, Marc'Aurelio Ranzato, Raia Hadsell,
  Maria{-}Florina Balcan, and Hsuan{-}Tien Lin, editors, \emph{Advances in
  Neural Information Processing Systems 33: Annual Conference on Neural
  Information Processing Systems 2020, NeurIPS 2020, December 6-12, 2020,
  virtual}, 2020.
\newblock URL
  \url{https://proceedings.neurips.cc/paper/2020/hash/a822554e5403b1d370db84cfbc530503-Abstract.html}.

\bibitem[Montasser et~al.(2021)Montasser, Hanneke, and
  Srebro]{montasser2021transductive}
Omar Montasser, Steve Hanneke, and Nathan Srebro.
\newblock Transductive robust learning guarantees.
\newblock \emph{arXiv preprint arXiv:2110.10602}, 2021.

\bibitem[Natarajan(1989)]{natarajan:89}
B.~K. Natarajan.
\newblock On learning sets and functions.
\newblock \emph{Machine Learning}, 4:\penalty0 67--97, 1989.

\bibitem[Rado(1949)]{rado1949axiomatic}
Richard Rado.
\newblock Axiomatic treatment of rank in infinite sets.
\newblock \emph{Canadian Journal of Mathematics}, 1\penalty0 (4):\penalty0
  337--343, 1949.

\bibitem[Rubinstein et~al.(2006)Rubinstein, Bartlett, and
  Rubinstein]{DBLP:conf/nips/RubinsteinBR06}
Benjamin I.~P. Rubinstein, Peter~L. Bartlett, and J.~Hyam Rubinstein.
\newblock Shifting, one-inclusion mistake bounds and tight multiclass expected
  risk bounds.
\newblock In Bernhard Sch{\"{o}}lkopf, John~C. Platt, and Thomas Hofmann,
  editors, \emph{Advances in Neural Information Processing Systems 19,
  Proceedings of the Twentieth Annual Conference on Neural Information
  Processing Systems, Vancouver, British Columbia, Canada, December 4-7, 2006},
  pages 1193--1200. {MIT} Press, 2006.
\newblock URL
  \url{https://proceedings.neurips.cc/paper/2006/hash/a11ce019e96a4c60832eadd755a17a58-Abstract.html}.

\bibitem[Schapire and Freund(2012)]{schapire:12}
R.~E. Schapire and Y.~Freund.
\newblock \emph{Boosting}.
\newblock Adaptive Computation and Machine Learning. MIT Press, Cambridge, MA,
  2012.

\bibitem[Shalev{-}Shwartz et~al.(2010)Shalev{-}Shwartz, Shamir, Srebro, and
  Sridharan]{DBLP:journals/jmlr/Shalev-ShwartzSSS10}
Shai Shalev{-}Shwartz, Ohad Shamir, Nathan Srebro, and Karthik Sridharan.
\newblock Learnability, stability and uniform convergence.
\newblock \emph{J. Mach. Learn. Res.}, 11:\penalty0 2635--2670, 2010.
\newblock URL \url{http://portal.acm.org/citation.cfm?id=1953019}.

\bibitem[Simon(1997)]{DBLP:journals/siamcomp/Simon97}
Hans~Ulrich Simon.
\newblock Bounds on the number of examples needed for learning functions.
\newblock \emph{{SIAM} J. Comput.}, 26\penalty0 (3):\penalty0 751--763, 1997.
\newblock \doi{10.1137/S0097539793259185}.
\newblock URL \url{https://doi.org/10.1137/S0097539793259185}.

\bibitem[Szegedy et~al.(2013)Szegedy, Zaremba, Sutskever, Bruna, Erhan,
  Goodfellow, and Fergus]{szegedy2013intriguing}
Christian Szegedy, Wojciech Zaremba, Ilya Sutskever, Joan Bruna, Dumitru Erhan,
  Ian Goodfellow, and Rob Fergus.
\newblock Intriguing properties of neural networks.
\newblock \emph{arXiv preprint arXiv:1312.6199}, 2013.

\bibitem[Vapnik(1982)]{vapnik:82}
V.~Vapnik.
\newblock \emph{Estimation of Dependencies Based on Empirical Data}.
\newblock Springer-Verlag, New York, 1982.

\bibitem[Vapnik and Chervonenkis(1971)]{vapnik:71}
V.~Vapnik and A.~Chervonenkis.
\newblock On the uniform convergence of relative frequencies of events to their
  probabilities.
\newblock \emph{Theory of Probability and its Applications}, 16\penalty0
  (2):\penalty0 264--280, 1971.

\bibitem[Vapnik and Chervonenkis(1974)]{vapnik:74}
V.~Vapnik and A.~Chervonenkis.
\newblock \emph{Theory of Pattern Recognition}.
\newblock Nauka, Moscow, 1974.

\bibitem[Wolk(1965)]{wolk1965note}
Elliot~S Wolk.
\newblock A note on" the comparability graph of a tree".
\newblock \emph{Proceedings of the American Mathematical Society}, 16\penalty0
  (1):\penalty0 17--20, 1965.

\bibitem[Zhang et~al.(2019)Zhang, Yu, Jiao, Xing, Ghaoui, and
  Jordan]{DBLP:conf/icml/ZhangYJXGJ19}
Hongyang Zhang, Yaodong Yu, Jiantao Jiao, Eric~P. Xing, Laurent~El Ghaoui, and
  Michael~I. Jordan.
\newblock Theoretically principled trade-off between robustness and accuracy.
\newblock In \emph{Proceedings of the 36th International Conference on Machine
  Learning, {ICML} 2019, 9-15 June 2019, Long Beach, California, {USA}},
  volume~97 of \emph{Proceedings of Machine Learning Research}, pages
  7472--7482. {PMLR}, 2019.
\newblock URL \url{http://proceedings.mlr.press/v97/zhang19p.html}.

\end{thebibliography}

\newpage
\appendix

\section{Proof of \prettyref{lem:orientation-to-prediction}}
\label{app:orientation-to-learner}

\begin{proof}
Let $\calO:E_n\to V_n$ be an arbitrary orientation of $G^{\calU}_\calH$. We will show that orientation $\calO$ implies a learner $\bbA_\calO:(\calX\times\calY)^{n-1}\to \calY^\calX$ with an expected robust risk $\calE_{n-1}$ that is upper bounded by the maximum adversarial out-degree of orientation $\calO$. 

We begin with describing the learner $\bbA_\calO$. For each input $((x_1,y_1),\dots,(x_{n-1},y_{n-1}), z) \in (\calX\times \calY)^{n-1}\times \calX$ define $\bbA_\calO((x_1,y_1),\dots,(x_{n-1},y_{n-1}))(z)$ as follows. Consider the set of vertices $v\in V$ that have the multiset $\SET{(x_1,y_1),\dots,(x_{n-1},y_{n-1})}$ and perturbation $z$ with a positive label
$$P_+= \SET{v \in V: \exists x \in \calX \text{ s.t. } z\in \calU(x) \wedge v=\SET{(x_1,y_1),\dots,(x_{n-1},y_{n-1}),(x,+1)}},$$
and the set of vertices $v\in V$ that have the multiset $\SET{(x_1,y_1),\dots,(x_{n-1},y_{n-1})}$ and perturbation $z$ with a negative label 
$$P_-=\SET{v \in V: \exists x \in \calX \text{ s.t. } z\in \calU(x) \wedge v=\SET{(x_1,y_1),\dots,(x_{n-1},y_{n-1}),(x,-1)}}.$$ 

We define $\bbA_\calO((x_1,y_1),\dots,(x_{n-1},y_{n-1}))(z)$ as a function of $P_+$, $P_-$, and the orientation $\calO$:
\[\bbA_\calO((x_1,y_1),\dots,(x_{n-1},y_{n-1}))(z)= \begin{cases}
y&~\text{if}~\inparen{\exists_{y\in \SET{\pm1}}}\inparen{\exists_{v\in P_{y}}}\inparen{\forall_{u\in P_{-y}}}: \calO((\SET{v,u}, z))=v.\\
+1&~\text{if}~P_+\neq \emptyset\wedge P_-= \emptyset.\\
-1&~\text{if}~P_+=\emptyset\wedge P_-\neq \emptyset.\\
+1&~\text{otherwise}.\\
\end{cases}\]
Note that $\bbA_\calO$ is well-defined. Specifically, observe that when $P_+\neq \emptyset$ and $P_-\neq \emptyset$, by definition of $P_+$ and $P_-$ and \prettyref{eqn:oneinclusion-edges-v2}, vertices from $P_+$ and $P_-$ form a complete bipartite graph. That is, for each $v\in P_+$ and each $u\in P_-$, $(\SET{u,v},z)\in E$. This implies that there exists at most one label: either $y=+1$ or $y=-1$ such that there is a vertex $v\in P_y$ where all edges $(\SET{v,u},z)\in E$ for $u\in P_{-y}$ are incident on $v$ according to orientation $\calO$: $\inparen{\exists! y\in \SET{\pm1}}\inparen{\exists v\in P_{y}}\inparen{\forall u\in P_{-y}}: \calO((\SET{v,u}, z))=v$.

We now proceed with bounding from above the expected robust risk $\calE_{n-1}$ of learner $\bbA_\calO$ by the maximum adversarial out-degree of orientation $\calO$. Consider an arbitrary multiset $S=\SET{(x_1,y_1),\dots,(x_n,y_n)}\in (\calX\times\calY)^n$ that is \emph{robustly} realizable \wrt~$(\calH,\calU)$. By definition $S\in V_n$, i.e., $S$ is a vertex in $G_\calH^{\calU}$. By definition of adversarial out-degree (see \prettyref{eqn:outdeg-v2}), there exists $T\subseteq S$ where $\abs{T} =  n - \outdeg(S;\calO)$ such that for each $(x,y)\in T$ and for each $z\in\calU(x)$: vertex $S$ will satisfy the condition that if there is any other vertex $u \in V_n$ where $(\SET{S,u}, z)$ is an edge: $(\SET{S,u}, z)\in E_n$, the orientation of edge $(\SET{S,u}, z)$ is towards $S$: $\calO((\SET{S,u},z))=S$. Thus by definition of $\bbA_{\calO}$ above, $\bbA_\calO\inparen{S\setminus\SET{(x,y)}, z}=y$. This implies that
\[\frac{1}{n} \sum_{i=1}^{n} \ind\insquare{\exists z\in\calU(x_{i}): \bbA_{\calO}\inparen{S\setminus\{(x_i,y_i)\}}(z) \neq y_i} = \frac{\outdeg(S; \calO)}{n}.\]
To conclude, by definition of $\calE_{n-1}(\bbA_{\calO};\calH,\calU)$ (see \prettyref{eqn:exp-rob-risk-A}),
\begin{align*}
    \calE_{n-1}(\bbA_{\calO};\calH,\calU) &= \sup_{\calD\in \RE(\calH,\calU)} \Ex_{S\sim \calD^{n-1}} \Risk_{\calU}(\bbA_{\calO}(S);\calD)\\
    &= \sup_{\calD\in \RE(\calH,\calU)} \Ex_{S\sim \calD^{n-1}} \Ex_{(x,y)\sim \calD} \ind\SET{\exists z\in \calU(x): \bbA_{\calO}(S)(z)\neq y}\\
    &= \sup_{\calD\in \RE(\calH,\calU)} \Ex_{S\sim \calD^{n}} \frac{1}{n}\sum_{i=1}^{n} \ind\SET{\exists z\in \calU(x_i): \bbA_{\calO}(S\setminus\{(x_i,y_i)\})(z)\neq y_i}\\
    &\leq \frac{\max_{v\in V_n}\outdeg(v;\calO)}{n}.
\end{align*}
\end{proof}

\section{Lemmas and Proofs for \prettyref{thm:realizable-bounds}}
\label{app:realizable-bounds}

\begin{lem} [Rado's Selection Principle \citep{rado1949axiomatic,wolk1965note}]
\label{lem:rado}
Let $I$ be an arbitrary index set, and let $\SET{X_i:i\in I}$ be a family of non-empty finite sets. For each finite subset $A$ of $I$, let $f_A$ be a choice function whose domain is $A$ and such that $f_A(i)\in X_i$ for each $i\in A$. Then, there exists a choice function $f$ whose domain is $I$ with the following property: to every finite subset $A$ of $I$ there corresponds a finite set $B$, $A\subseteq B\subseteq I$, with $f(i)=f_B(i)$ for each $i\in A$.
\end{lem}

\begin{lem} [Refined Lowerbound on Error Rate of Learners]
\label{lem:lowerbound-eps}
For any integer $n\geq 1$, let $G_\calH^\calU=(V_{2n},E_{2n})$ be the global one-inclusion graph as defined in \prettyref{eqn:oneinclusion-vertices-v2} and \prettyref{eqn:oneinclusion-edges-v2}. Then, for any learner $\bbA:(\calX\times \calY)^*\to \calY^\calX$ and any $\eps\in (0,1)$, there exists an orientation $\calO_\bbA:E_{2n}\to V_{2n}$ of $G^\calU_{\calH}$ such that $$\calE_{\frac{n}{\epsilon}}(\bbA;\calH,\calU)\geq \frac{\eps}{2} \cdot \frac{\max_{v\in V_{2n}} \outdeg(v;\calO_\bbA) - 1}{n-1}.$$
\end{lem}

\begin{proof} [of \prettyref{lem:lowerbound-eps}]
Set $m=\frac{n}{\eps}$. We begin with describing the orientation $\calO_\bbA$ by orienting edges incident on each vertex $v\in V_{2n}$. Consider an arbitrary vertex $v=\SET{(x_1,y_1),\dots, (x_{2n},y_{2n})}$. Without loss of generality, let $P_v$ be a distribution over $\SET{(x_1,y_1),\dots, (x_{2n},y_{2n})}$, defined as
\[P_v(\SET{(x_1,y_1)}) = 1-\eps \text{ and } P_v(\SET{(x_t,y_t)})=\frac{\eps}{2n-1}~~\forall 2\leq t \leq 2n.\]

For each $1\leq t \leq 2n$, let $$p_t(v)=\Prob_{S\sim P_v^m} \insquare{\exists z\in \calU(x_t): \bbA(S)(z)\neq y_t | (x_t,y_t)\notin S}.$$
For each $1\leq t \leq 2n$ such that $(x_t,y_t)\in v$ witnesses an edge, i.e. $\exists u \in V_{2n}, z\in \calX\text{ s.t. }(\SET{v, u}, z)\in E_{2n}$ and $(x_t,y_t)\in v\Delta u$, if $p_t < \frac{1}{2}$, then orient \emph{all} edges incident on $(x_t,y_t)$ inward, otherwise orient them arbitrarily. Note that this might yield edges that are oriented outwards from both their endpoint vertices, in which case, we arbitrarily orient such an edge. Observe also that we will not encounter a situation where edges are oriented inwards towards both their endpoints (which is an invalid orientation). This is because for any two vertices $v,u\in V_{2n}$ such that $\exists z_0\in \calX$ where $(\SET{u,v},z_0)\in E_{2n}$ and $v \Delta u = \SET{(x_t,y_t),(\Tilde{x}_t,-y_t)}$, we can not have $p_{t}(v) < \frac{1}{2}$ and $p_t(u) < \frac{1}{2}$,
since
\[p_t(v) \geq \Prob_{S\sim P_v^m}\insquare{\bbA(S)(z_0)\neq y_t | (x_t,y_t)\notin S}~~\text{and}~~p_t(u) \geq \Prob_{S\sim P_u^m}\insquare{\bbA(S)(z_0)\neq -y_t | (\Tilde{x}_t,-y_t)\notin S},\]
and $P_v$ conditioned on $(x_t,y_t)\notin S$ is the same distribution as $P_u$ conditioned on $(\Tilde{x}_t,-y_t)\notin S$. This concludes describing the orientation $\calO_\bbA$. We now bound from above the adversarial out-degree of vertices $v\in V_{2n}$ \wrt the orientation $\calO_\bbA$:
\begin{align*}
    \outdeg(v;\calO_\bbA)&\leq \sum_{t=1}^{2n}\ind\insquare{p_t\geq \frac{1}{2}} \leq 1 + \sum_{t=2}^{2n}\ind\insquare{p_t\geq \frac{1}{2}} \leq 1 + 2\sum_{t=2}^{2n} p_t\\
    &=1+ 2\sum_{t=2}^{2n} \Prob_{S\sim P^m} \insquare{\exists z\in \calU(x_t): \bbA(S)(z)\neq y_t | (x_t,y_t)\notin S}\\
    &=1+ 2 \sum_{t=2}^{2n} \frac{\Prob_{S\sim P^m} \insquare{\inparen{\exists z\in \calU(x_t): \bbA(S)(z)\neq y_t} \wedge (x_t,y_t)\notin S}}{\Prob_{S\sim P^m}\insquare{(x_t,y_t)\notin S}}\\
    &\overset{(i)}{\leq}1+ 2\cdot \inparen{1-\frac{n}{2n-1}} \sum_{t=2}^{2n} \Prob_{S\sim P^m} \insquare{\inparen{\exists z\in \calU(x_t): \bbA(S)(z)\neq y_t} \wedge (x_t,y_t)\notin S}\\
    &= 1+ 2\cdot \inparen{1-\frac{n}{2n-1}} \sum_{t=2}^{2n} \Ex_{S\sim P^m} \insquare{\ind\insquare{\exists z\in \calU(x_t): \bbA(S)(z)\neq y_t}\ind\insquare{(x_t,y_t)\notin S}}\\
    &= 1+ 2\cdot \inparen{1-\frac{n}{2n-1}} \Ex_{S\sim P^m} \insquare{\sum_{t=2}^{2n} \ind\insquare{\exists z\in \calU(x_t): \bbA(S)(z)\neq y_t}\ind\insquare{(x_t,y_t)\notin S}}\\
    &\leq 1+ 2\cdot \inparen{1-\frac{n}{2n-1}} \Ex_{S\sim P^m} \insquare{\sum_{t=2}^{2n} \ind\insquare{\exists z\in \calU(x_t): \bbA(S)(z)\neq y_t} }\\
    &= 1+ 2\cdot \inparen{1-\frac{n}{2n-1}} \cdot \frac{2n-1}{\eps} \Ex_{S\sim P^m} \insquare{\frac{\eps}{2n-1} \sum_{t=2}^{2n} \ind\insquare{\exists z\in \calU(x_t): \bbA(S)(z)\neq y_t} }\\
    &\leq 1+ 2\cdot \inparen{1-\frac{n}{2n-1}} \cdot \frac{2n-1}{\eps} \Ex_{S\sim P^m} \Risk_\calU(\bbA(S);P)\\
    &\leq 1+ 2\cdot \inparen{1-\frac{n}{2n-1}} \cdot \frac{2n-1}{\eps} \calE_m(\bbA;\calH,\calU) = 1 + \frac{2(n-1)}{\eps} \calE_m(\bbA;\calH,\calU),
\end{align*}
where inequality $(i)$ follows from the following:
\[\Prob_{S\sim P^m}\insquare{(x_t,y_t)\notin S} = \inparen{1-\frac{\eps}{2n-1}}^m\geq 1 - m\cdot \frac{\eps}{2n-1} = 1-\frac{n}{\eps}\frac{\eps}{2n-1} \geq 1-\frac{n}{2n-1}.\]
Since the above holds for any vertex $v\in V_{2n}$, by rearranging terms, we get $\calE_m(\bbA;\calH,\calU)\geq \frac{\eps}{2}\frac{\max_{v\in V_{2n}} \outdeg(v;\calO_\bbA) - 1}{n -1}$.
\end{proof}

\begin{proof}[of \prettyref{lem:expectation-bounds}]
We will first start with the upper bound. Let $n>\CM$ and let $G_\calH^\calU=(V_n,E_n)$ be the (possibly infinite) one-inclusion graph. Then, by definition of $\CM$, for every finite subgraph $G=(V,E)$ of $G_\calH^\calU$ there exists an orientation $\calO_E:E\to V$ such that every vertex in the subgraph has adversarial out-degree at most $\frac{n}{3}$: $\forall v\in V, \outdeg(v;\calO_E)\leq \frac{n}{3}$.

We next invoke \prettyref{lem:rado} where $E_n$ represents our family of non-empty finite sets, and for each finite subset $E\subseteq E_n$, we let the orientation $\calO_E$ (from above) represent the choice function. Then, \prettyref{lem:rado} implies that there exists an orientation $\calO:E_n\to V_n$ of $G_{\calH}^{\calU}$ (i.e., an orientation of the entire global one-inclusion graph) with the following property: for each finite subset $A$ of $E_n$, there corresponds a finite set $E$ satisfying $A\subseteq E \subseteq E_n$ and $\calO(e) =\calO_E(e)$ for each $e\in A$. This implies that orientation $\calO$ satisfies the property that $\forall v\in V_n, \outdeg(v;\calO)\leq \frac{n}{3}$. Because, if not, then we can find a subgraph $G=(E,V)$ where $\calO_E$ (from above) violates the adversarial out-degree upper bound of $\frac{n}{3}$ and that leads to a contradiction.

Now, we use orientation $\calO$ of $G_\calH^\calU$ (which has adversarial out-degree at most $\frac{n}{3}$) to construct a learner $\bbA_{\calO}:(\calX\times \calY)^{n-1}\times \calX \to \calY$ as in \prettyref{lem:orientation-to-prediction}. Then, \prettyref{lem:orientation-to-prediction} implies that
\[\calE_{n-1}(\calH,\calU)\leq \calE_{n-1}(\bbA_\calO;\calH,\calU)\leq \frac{1}{3}.\]

We now turn to the lower bound. Let $2\leq n\leq \frac{\CM}{2}$, $\eps\in(0,1)$, and let $G_\calH^\calU=(V_{2n},E_{2n})$ be the (possibly infinite) one-inclusion graph. Since $2n\leq \CM$, by definition of $\CM$, it follows that there exists a finite subgraph $G=(V,E)$ of $G_\calH^\calU=(V_{2n},E_{2n})$ such that 
\begin{equation}
    \label{eqn:subgraph-outdeg}
    \forall \text{ orientations } \calO:E\to V \text{ of subgraph }G, \max_{v\in V} \outdeg(v;\calO)\geq \frac{2n}{3}.
\end{equation}

Now, let $\bbA:(\calX\times \calY)^* \to \calY^\calX$ be an arbitrary learner. We invoke \prettyref{lem:lowerbound-eps}, which is a refined statement of \prettyref{lem:lowerbound-inductive} that takes $\eps$ into account, to orient the subgraph $G$ using learner $\bbA$. \prettyref{lem:lowerbound-eps} and \prettyref{eqn:subgraph-outdeg} above imply that
\begin{align*}
    \calE_\frac{n}{\eps}(\calH,\calU)&\geq \calE_{\frac{n}{\eps}}(\bbA;\calH,\calU)\geq \frac{\eps}{2}\frac{\max_{v\in V} \outdeg(v;\calO_\bbA)-2}{n-1} \geq \frac{\eps}{2} \frac{(2n)/3-1}{n-1}\\
    &=\frac{\eps}{3}\frac{2n-2-1}{2n-2}=\frac{\eps}{3}\inparen{1-\frac{1}{2n-2}}\geq \frac{\eps}{6}.
\end{align*}
\end{proof}

\begin{lem} [Sample Compression Robust Generalization -- \cite{pmlr-v99-montasser19a}]
\label{lem:robust-compression}
For any $k \in \bbN$ and fixed function $\phi : (\calX \times \calY)^{k} \to \calY^{\calX}$, for any distribution $P$ over $\calX \times \calY$ and any $m \in \bbN$, 
for $S = \{(x_{1},y_{1}),\ldots,(x_{m},y_{m})\}$ iid $P$-distributed random variables,
with probability at least $1-\delta$, 
if $\exists i_{1},\ldots,i_{k} \in \{1,\ldots,m\}$ 
s.t.\ $\hat{R}_{\calU}(\phi((x_{i_{1}},y_{i_{1}}),\ldots,(x_{i_{k}},y_{i_{k}}));S) = 0$, 
then 
\begin{equation*}
\Risk_{\calU}(\phi((x_{i_{1}},y_{i_{1}}),\ldots,(x_{i_{k}},y_{i_{k}}));P) \leq \frac{1}{m-k} (k\ln(m) + \ln(1/\delta)).
\end{equation*}
\end{lem}

We are now ready to proceed with the proof of \prettyref{thm:realizable-bounds}.

\begin{proof}[of \prettyref{thm:realizable-bounds}]
We begin with proving the upper bound. Let $m_0=\CM$. By \prettyref{lem:expectation-bounds}, there exists a learner $\bbA$ from orienting the \graph~$G_{\calH}^{\calU}=(V_{m_0+1},E_{m_0+1})$ that satisfies worst-case expected risk
\begin{equation}
\label{eqn:in-proof}
    \calE_{m_0}(\bbA;\calH,\calU)\leq \frac{1}{3}.
\end{equation}
Let $\calD\in\RE(\calH,\calU)$ be some unknown \emph{robustly realizable} distribution. Fix $\eps,\delta\in(0,1)$ and a sample size $m(\eps,\delta)$ that will be determined later. Let $S=\SET{(x_1,y_1),\dots,(x_m,y_m)}$ be an i.i.d.~sample from $\calD$. Our strategy is to use $\bbA$ above as a \emph{weak} robust learner and boost its confidence and robust error guarantee.

\paragraph{Weak Robust Learner.} Observe that, by \prettyref{eqn:in-proof}, for any empirical distribution $D$ over $S$, $\Ex_{S'\sim D^{m_0}} \Risk_{\calU}(\bbA(S');D)\leq 1/3$.
This implies that for any empirical distribution $D$ over $S$, there exists at least one sequence $S_D\in (S)^{m_0}$ such that $h_D:=\bbA(S_D)$ satisfies $\Risk_\calU(h_{D};D)\leq 1/3$. We use this to define a weak robust-learner for 
distributions $D$ over $S$: 
i.e., for any $D$, the weak learner chooses $h_{D}$ 
as its weak hypothesis.

\paragraph{Boosting.}
Now we run the $\alpha$-Boost boosting algorithm 
\citep*[][Section 6.4.2]{schapire:12} 
on data set $S$, 
but using the robust loss rather than $0$-$1$ loss.  
That is, we start with $D_{1}$ uniform on $S$. Then for each round $t$, we get $h_{D_{t}}$ as a weak robust classifier with respect to $D_{t}$, 
and for each $(x,y) \in S$ we define 
a distribution $D_{t+1}$ over $S$ satisfying 
\[D_{t+1}(\SET{(x,y)}) = \frac{D_{t}(\SET{(x,y)})}{Z_{t}} \times \begin{cases} 
                                                              e^{-2\alpha}& \text{if }\ind[\forall z\in \calU(x): h_{D_{t}}(z) = y]=1\\
                                                              1 &\text{otherwise}
                                                           \end{cases}
\]
where $Z_{t}$ is a normalization factor, $\alpha$ is a parameter that will be determined below.
Following the argument from \citep*[][Section 6.4.2]{schapire:12}, after $T$ rounds we are guaranteed 
\begin{equation*}
\min_{(x,y) \in S} \frac{1}{T} \sum_{t=1}^{T} \ind[ \forall z \in \calU(x): h_{D_t}(z)=y ]
\geq \frac{2}{3} - \frac{2}{3}\alpha - \frac{\ln(|S|)}{2\alpha T},
\end{equation*}
so we will plan on running until round 
$T = 1 + 48 \ln(|S|)$ 
with value 
$\alpha = 1/8$ 
to guarantee
\begin{equation*}
\min_{(x,y) \in S} \frac{1}{T} \sum_{t=1}^{T} \ind[ \forall z \in \calU(x): h_{D_t}(z)=y ]
> \frac{1}{2},
\end{equation*}
so that the majority-vote classifier $\MAJ(h_{D_1},\dots,h_{D_T})$ achieves \emph{zero} robust loss on the empirical dataset $S$, $\Risk_{\calU}(\MAJ(h_{D_1},\dots,h_{D_T});S)=0$. 

Furthermore, note that, since each $h_{D_{t}}$
is given by $\bbA(S_{D_{t}})$, where $S_{D_{t}}$ is an 
$m_0$-tuple of points in $S$, 
the classifier $\MAJ(h_{D_1},\dots,h_{D_T})$ is specified by an 
ordered sequence of $m_0\cdot T$ points from $S$. Thus, the classifier $\MAJ(h_1,\dots,h_T)$ is representable as the value of an (order-dependent) reconstruction function $\phi$ with 
a compression set size $m_0T=m_0O(\log m)$. Now, invoking \prettyref{lem:robust-compression}, we get the following \emph{robust} generalization guarantee: with probability at least $1-\delta$ over $S\sim \calD^{m}$,
\[\Risk_{\calU}(\MAJ(h_1,\dots,h_T);\calD)\leq O\inparen{\frac{m_0\log^2 m}{m} + \frac{\log(1/\delta)}{m}},\]
and setting this less than $\eps$ and solving for a sufficient size of $m$ yields the stated sample complexity bound. 

We now turn to proving the lower bound. Let $n_0=\frac{\CM}{2}$, by invoking \prettyref{lem:expectation-bounds}, we get that $\calE_{n_0/\eps}\geq \Omega(\eps)$. Then, by \prettyref{eqn:sample-complexity}, this implies that $\calM_{\eps}(\calH,\calU)\geq \Omega(1/\eps)n_0 \geq \Omega(1/\eps)\CM$. 
\end{proof}

\section{Proofs for \prettyref{sec:agnostic}}
\label{app:section-5}

\begin{defn}[Agnostic Robust PAC Learnability]
\label{def:ag_sample_complexity}
For any $\eps,\delta \in (0,1)$, the \emph{sample complexity of agnostic robust $(\eps, \delta)-$PAC learning of $\calH$ with respect to perturbation set $\calU$}, denoted $\calM^{\ag}_{\eps,\delta}(\calH,\calU)$, is defined as the smallest $m \in \bbN \cup \{0\}$ for which there exists a learner $\bbA: (\calX\times \calY)^* \to \calY^\calX$ such that, for every data distribution $\calD$ over $\calX \times \calY$, with probability at least $1-\delta$ over $S \sim \calD^m$,
\begin{equation*}
\Risk_{\calU}(\bbA(S);\calD)\leq \inf\limits_{h \in \calH} \Risk_{\calU}(h;\calD) + \eps.
\end{equation*}
If no such $m$ exists, define $\calM^{\ag}_{\eps,\delta}(\calH,\calU) = \infty$. We say that $\calH$ is robustly PAC learnable in the agnostic setting with respect to perturbation set $\calU$ if $\forall \epsilon,\delta \in (0,1)$, 
$\calM^{\ag}_{\eps,\delta}(\calH,\calU)$ is finite.
\end{defn}

\begin{proof}[of \prettyref{lem:reduction-to-realizable}]
The argument follows closely a proof of an analogous result by \citep*{david:16} for non-robust learning, and \citep*{pmlr-v99-montasser19a} for robust learning.
Denote by $\bbA$ a
realizable learner with sample complexity 
$\calM^\re_{1/3}(\calH,\calU)$, and denote $\Mre = \calM^\re_{1/3}(\calH,\calU)$. 

\paragraph{Description of agnostic learner $\bbB$.} Given a data set $S \sim \calD^{m}$ where $\calD$ is some unknown distribution, 
we first do robust-ERM to find a maximal-size 
subsequence $S^{\prime}$ of the data where the robust 
loss can be zero: that is, 
$\inf_{h \in \calH} \hat{\Risk}_{\calU}(h;S^{\prime}) = 0$.
Then for any distribution $D$ over $S^{\prime}$, 
there exists a sequence $S_{D} \in (S^{\prime})^{\Mre}$ 
such that $h_{D} := \bbA(S_{D})$ has 
$\Risk_{\calU}(h_{D};D) \leq 1/3$; 
this follows since, by definition of 
$\calM^\re_{1/3}(\calH,\calU)$, $\calE_{\Mre}(\bbA;\calH,\calU)\leq 1/3$ so at least one such $S_{D}$ exists.
We use this to define a weak robust-learner for 
distributions $D$ over $S^{\prime}$: 
i.e., for any $D$, the weak learner chooses $h_{D}$ 
as its weak hypothesis.

Now we run the $\alpha$-Boost boosting algorithm 
\citep*[][Section 6.4.2]{schapire:12} 
on data set $S^{\prime}$, 
but using the robust loss rather than $0$-$1$ loss.  
That is, we start with $D_{1}$ uniform on $S^{\prime}$.\footnote{We ignore the possibility of repeats; 
for our purposes we can just remove any repeats from $S^{\prime}$ 
before this boosting step.}
Then for each round $t$, we get $h_{D_{t}}$ 
as a weak robust classifier with respect to $D_{t}$, 
and for each $(x,y) \in S^{\prime}$ we define 
a distribution $D_{t+1}$ over $S^{\prime}$ satisfying 
\begin{equation*}
D_{t+1}(\{(x,y)\}) \propto 
D_{t}(\{(x,y)\}) \exp\!\left\{-2\alpha \ind[ \forall x^{\prime} \in \calU(x), h_{D_{t}}(x^{\prime})=y ] \right\},
\end{equation*}
where $\alpha$ is a parameter we can set.
Following the argument from \citep*[][Section 6.4.2]{schapire:12}, after $T$ rounds we are guaranteed 
\begin{equation*}
\min_{(x,y) \in S^{\prime}} \frac{1}{T} \sum_{t=1}^{T} \ind[ \forall x^{\prime} \in \calU(x), h_{D_{t}}(x^{\prime})=y ]
\geq \frac{2}{3} - \frac{2}{3}\alpha - \frac{\ln(|S^{\prime}|)}{2\alpha T},
\end{equation*}
so we will plan on running until round 
$T = 1 + 48 \ln(|S^{\prime}|)$ 
with value 
$\alpha = 1/8$ 
to guarantee
\begin{equation*}
\min_{(x,y) \in S^{\prime}} \frac{1}{T} \sum_{t=1}^{T} \ind[ \forall x^{\prime} \in \calU(x), h_{D_{t}}(x^{\prime})=y ]
> \frac{1}{2},
\end{equation*}
so that the classifier 
$\hat{h}(x) := \ind\!\left[ \frac{1}{T} \sum_{t=1}^{T} h_{D_{t}}(x) \geq \frac{1}{2} \right]$ 
has $\hat{\Risk}_{\calU}(\hat{h};S^{\prime}) = 0$.

Furthermore, note that, since each $h_{D_{t}}$
is given by $\bbA(S_{D_{t}})$, where $S_{D_{t}}$ is an 
$\Mre$-tuple of points in $S^{\prime}$, 
the classifier $\hat{h}$ is specified by an 
ordered sequence of $\Mre T$ points from $S$.
Altogether, $\hat{h}$ is a function specified 
by an ordered sequence of $\Mre T$ points 
from $S$, and which has 
\begin{equation*}
\hat{\Risk}_{\calU}(\hat{h};S) \leq \min_{h \in \calH} \hat{\Risk}_{\calU}(h;S).
\end{equation*}
Similarly to the realizable case (see the proof of Lemma~\ref{lem:robust-compression}), 
uniform convergence guarantees for sample compression 
schemes \citep*[see][]{graepel:05} remain valid for the robust loss, 
by essentially the same argument; the essential argument is the 
same as in the proof of Lemma~\ref{lem:robust-compression} except 
using Hoeffding's inequality to get concentration of the empirical 
robust risks for each fixed index sequence, and then a union bound over 
the possible index sequences as before.
We omit the details for brevity. 
In particular, denoting $T_{m} = 1 + 48 \ln(m)$, 
for $m > \Mre T_{m}$, with probability at least $1-\delta/2$, 
\begin{equation*}
\Risk_{\calU}(\hat{h};\calD)
\leq \hat{\Risk}_{\calU}(\hat{h};S) 
+ \sqrt{\frac{\Mre T_{m} \ln(m) + \ln(2/\delta)}{2m - 2\Mre T_{m}}}.
\end{equation*}

Let $h^{*} = \argmin_{h \in \calH} \Risk_{\calU}(h;\calD)$ 
(supposing the min is realized, for simplicity; else 
we could take an $h^{*}$ with very-nearly minimal risk). 
By Hoeffding's inequality, with probability at least 
$1-\delta/2$, 
\begin{equation*}
\hat{\Risk}_{\calU}(h^{*};S) 
\leq \Risk_{\calU}(h^{*};\calD) 
+ \sqrt{\frac{\ln(2/\delta)}{2m}}.
\end{equation*}

By the union bound, if $m \geq 2 \Mre T_{m}$, with probability at least $1-\delta$, 
\begin{align*}
\Risk_{\calU}(\hat{h};\calD) 
& \leq \min_{h \in \calH} \hat{\Risk}_{\calU}(h;S) 
+ \sqrt{\frac{\Mre T_{m} \ln(m) + \ln(2/\delta)}{m}}
\\ & \leq \hat{\Risk}_{\calU}(h^{*};S) 
+ \sqrt{\frac{\Mre T_{m} \ln(m) + \ln(2/\delta)}{m}}
\\ & \leq \Risk_{\calU}(h^{*};\calD) 
+ 2\sqrt{\frac{\Mre T_{m} \ln(m) + \ln(2/\delta)}{m}}.
\end{align*}
Since 
$T_{m} = O( \log(m) )$, 
the above is at most $\eps$ 
for an appropriate choice of sample size 
$m = O\!\left( \frac{\Mre}{\eps^{2}} \log^{2}\!\left(\frac{\Mre}{\eps}\right) + \frac{1}{\eps^{2}}\log\!\left(\frac{1}{\delta}\right) \right)$.
This concludes the upper bound on $\calM^{\ag}_{\eps,\delta}(\bbB;\calH,\calU)$, and the lower bound trivially holds from the definition of $\bbB$. 
\end{proof}

\section{Finite Character Property}
\label{app:finite-char-property}

\citet*{DBLP:journals/natmi/Ben-DavidHMSY19} gave a \emph{formal} definition of the notion of ``dimension'' or ``complexity measure'', that all previously proposed dimensions in statistical learning theory comply with. In addition to characterizing learnability, a dimension should satisfy the \emph{finite character} property:
\begin{defn}[Finite Character]
\label{def:complexity-measure}
A dimension characterizing learnability can be abstracted as a function $F$ that maps a class $\calH$ to $\bbN\cup \SET{\infty}$ and satisfies the \emph{finite character} property: For every $d\in\bbN$ and $\calH$, the statement $``F(\calH)\geq d"$ can be demonstrated by a finite set $X\subseteq \calX$ of domain points, and a finite set of hypotheses $H\subseteq \calH$. That is, $``F(\calH)\geq d"$ is equivalent to the existence of a bounded first order formula $\phi(\calX,\calH)$ in which all the quantifiers are of the form: $\exists x\in\calX,\forall x\in\calX$ or $\exists h\in\calH,\forall h\in \calH$.
\end{defn}
For example, the property $``\vc(\calH)\geq d"$ is a finite character property since it can be verified with a finite set of points $x_1\dots, x_d\in \calX$ and a finite set of classifiers $h_{1},\dots, h_{2^d}\in \calH$ that shatter these points, and a predicate $E(x,h)\equiv x \in h$ (i.e., the value $h(x)$). In our case, in addition to having a domain $\calX$ and a hypothesis class $\calH$, we also have a relation $\calU$. In \prettyref{clm:finite-char}, we argue that our dimension $\CM$ satisfies \prettyref{def:complexity-measure}, though unlike VC dimension, we do need $\forall$ quantifiers. Furthermore, we provably \emph{cannot} verify the statement $\CM\geq d$ by evaluating the predicate $E(x,h)$ on finitely many $x$'s and $h$'s, but we can verify it using a predicate $P_{\calU}(x,h)\equiv  \forall z\in\calU(x): h(z)=h(x)$ that evaluates the \emph{robust} behavior of $h$ on $x$ w.r.t.~$\calU$. The proof is deferred to \prettyref{app:finite-char-property}.
\begin{claim}
\label{clm:finite-char}
$\CM$ satisfies the finite character property of \prettyref{def:complexity-measure}.
\end{claim}

\begin{proof}
By the definition of $\CM$ in \prettyref{eqn:complexity-measure}, to demonstrate that $\CM \geq d$, it suffices to present a finite subgraph $G=(V,E)$ of $G^\calU_\calH=(V_d,E_d)$ where every orientation $\calO:E\to V$ has adversarial out-degree at least $\frac{n}{3}$. Since $V$ is, by definition, a finite collection of datasets robustly realizable \wrt $(\calH,\calU)$ this means that we can demonstrate that $\CM \geq d$ with a finite set $X\subseteq \calX$ and a finite set of hypotheses $H\subseteq \calH$ that can construct the finite collection $V$. 

Note that in our case, we do not only have $\calX$ and $\calH$, but also a set relation $\calU$ that specifies for each $x\in\calX$ its corresponding set of perturbations $\calU(x)$. We can still express $\CM \geq d$ with a bounded formula using only quantifiers over $\calH$ and $\calX$, though unlike in the case of VC dimension, we do also need $\forall$ quantifiers.
Furthermore, we provably cannot verify the formula by evaluating $h(x)$ only on finitely many $x\in\calX, h\in \calH$, since $\calU(x)$ can be \emph{infinite}. But, we can verify it given access to a predicate $P_\calU(h,x)\equiv \forall z\in\calU(x): h(z)=h(x)$. 
\end{proof}
\end{document}